\documentclass[lettersize,journal]{IEEEtran}
\usepackage{multirow}
\usepackage{hhline}
\usepackage{mathtools, cuted}
\usepackage[strict]{changepage}
\usepackage{epsfig}
\usepackage{amsmath}
\usepackage{amssymb}
\usepackage{color}
\usepackage[dvipsnames]{xcolor}
\usepackage{balance}
\usepackage{cite}
\usepackage{times}
\usepackage{graphicx}
\usepackage{psfrag}
\usepackage{array}
\usepackage{times}
\usepackage{latexsym,theorem}
\usepackage{wrapfig}
\usepackage{hyperref}
\usepackage{colortbl}
\usepackage{amsmath,amsfonts,amssymb,amsbsy}
\usepackage{tikz}
\usetikzlibrary{decorations.shapes,shapes.geometric,shadows,arrows,automata,positioning,calendar,mindmap,backgrounds,scopes,chains,er,3d,matrix,fit}

\newcommand{\R}{\mathbb{R}}
\newcommand{\B}{\mathbb{B}}

\newcommand{\cS}{\mathcal{S}}
\newcommand{\cU}{\mathcal{U}}
\newcommand{\cX}{\mathcal{X}}
\newcommand{\cQ}{\mathcal{Q}}
\newcommand{\cC}{\mathcal{C}}

\newcommand{\cB}{\mathcal{B}}

\newcommand{\cV}{\mathcal{V}}
\newcommand{\cF}{\mathcal{F}}
\newcommand{\cR}{\mathcal{R}}

\newcommand{\cZ}{\mathcal{Z}}
\newcommand{\cI}{\mathcal{I}}

\newcommand{\cW}{\mathcal{W}}
\newcommand{\cY}{\mathcal{Y}}

\newcommand{\interior}{\operatorname{int}}
\newcommand{\boundary}[1]{\operatorname{bd}(#1)}

\newtheorem{lemma}{Lemma}
\newtheorem{definition}{Definition}
\newtheorem{theorem}{Theorem}
\newtheorem{proposition}{Proposition}
\newtheorem{corollary}{Corollary}
\newtheorem{assumption}{Assumption}

\newtheorem{Fact}{Fact}
\newtheorem{ilex}{Illustrative Example}

\IEEEoverridecommandlockouts                            

\newcommand{\argmin}[1]{\underset{#1}{\operatorname{arg}\,\operatorname{min}}\;}

\usepackage{algorithm}
\usepackage[noend]{algpseudocode}
\usepackage{subcaption}

\newcommand{\cT}{\mathcal{T}}

\algnewcommand{\algorithmicand}{\textbf{ and }}
\algnewcommand{\algorithmicor}{\textbf{ or }}
\algnewcommand{\OR}{\algorithmicor}
\algnewcommand{\AND}{\algorithmicand}
\algnewcommand{\var}{\texttt}
\usepackage{svg}

\usepackage{psfrag}

\newtheorem{myexp}{Example}

\usepackage{authblk}

\definecolor{constraint_color}{rgb}{0.1,0.55,1.0}

\tikzset{rfill/.code={%
\pgfmathsetmacro\r{rnd}\pgfmathsetmacro\g{rnd}\pgfmathsetmacro\b{rnd}%
\definecolor{.}{rgb}{\r,\g,\b}%
\pgfsetfillcolor{.}%
}}

\begin{document}
\title{Set based velocity shaping for robotic manipulators}

\markboth{Journal of \LaTeX\ Class Files,~Vol.~14, No.~8, August~2015}%
{Shell \MakeLowercase{\textit{et al.}}: Bare Demo of IEEEtran.cls for IEEE Journals}
\author{Ryan McGovern}
\author{Nikolaos~Athanasopoulos}
\author{Se\'{a}n McLoone}
\affil{School of Electronics, Electrical Engineering and Computer Science, Queen's University Belfast, Northern Ireland, UK. E-mail: \texttt{ \{rmcgovern03,n.athanasopoulos,s.mcloone\}@qub.ac.uk }}
\maketitle
\pagestyle{empty}
\thispagestyle{empty}

\let\thefootnote\relax\footnotetext{R.M. gratefuly acknowledges support from the UK  DfE and N.A. from EU 2020-1-UK01-KA203-079283 and EPSRC EP/T021942/1. }

\begin{abstract}
We develop a new framework for trajectory planning on predefined paths, for general N-link manipulators.
Different from previous approaches generating open-loop minimum time controllers or pre-tuned motion profiles by time-scaling, we establish analytic algorithms that recover all initial conditions that can be driven to the desirable target set while adhering to environment constraints. More technologically relevant, we characterise families of corresponding safe state-feedback controllers with several desirable properties. A key enabler in our framework is the introduction of a state feedback template, that induces ordering properties between trajectories of the resulting closed-loop system. The proposed structure allows working on the nonlinear system directly in both the analysis and synthesis problems. Both  offline computations and online implementation are  scalable with respect to the number of links of the manipulator. The results can potentially be used in a series of challenging problems: Numerical experiments on a commercial robotic manipulator demonstrate that efficient online implementation is possible. 
\end{abstract}

\section{Introduction}\label{section1}
\begin{figure*}
    \centering
    \begin{subfigure}[b]{0.3\textwidth}        \includegraphics[width=\textwidth]{"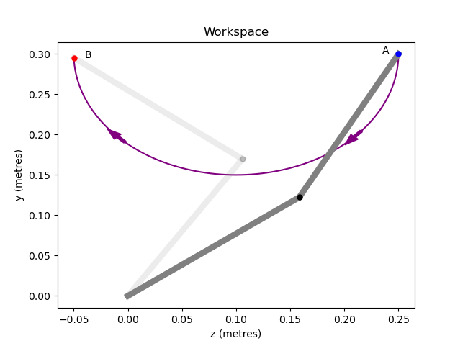"}
        \caption{}
        \label{fig-a}
    \end{subfigure}
    \begin{subfigure}[b]{0.3\textwidth}
        \includegraphics[width=\textwidth]{"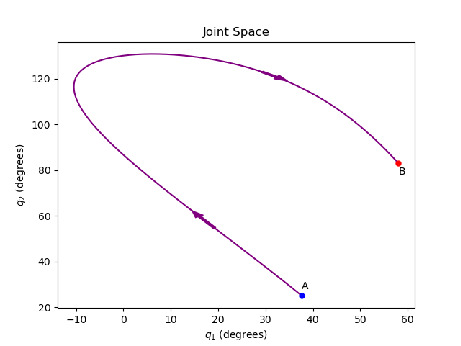"}
        \caption{}
        \label{fig-b}
    \end{subfigure}
    \begin{subfigure}[b]{0.3\textwidth}
        \includegraphics[width=\textwidth]{"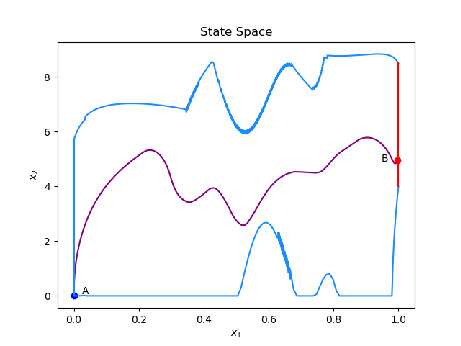"}
        \caption{}
        \label{fig-c}
    \end{subfigure}
  \caption{(a) Workspace representation of a path for a two DOF robot. (b) Joint space representation of the two DOF manipulator. (c) State space representation of the projeced path dynamics. The constraints curve is in blue, and the target set is in red.} 
  \label{fig:main}
\end{figure*}

\noindent
In recent years, there has been growing demand for safer robots, especially in the context of Industry 4.0 and environments containing fragile agents such as humans \cite{villani2018survey}. Robots face challenges in their integration with the workforce\cite{michaelis2020collaborative}, in particular  related to safety, inclusion of temporal specifications, and reliability \cite{alegue2018human}, \cite{kunze2018artificial}.

Planning of robotic manipulators has received much attention over the years with a wide range of approaches available \cite{lavalle_path_planning}. The challenge lies in computing collision free paths in a nonconvex environment, and in a sufficiently fast time \cite{volz2019predictive}. 

Established approaches to producing safe controllers aim to place constraints on speed, momentum and potential collision energies of the system in order to ensure safety \cite{joseph2018towards,joseph2018experimental,rossi2015pre}. These controllers rely on the trajectory planner found within the typical motion controller.

Trajectory planning, which is the focus of this article, is concerned with shaping the velocity profile as the robot moves from one configuration to another. Typically, the enforcement of hard constraints on the system dynamics is imposed on the velocity profile, see, e.g., \cite{Shin_McKay:85,bobrow1985time,Pfeiffer_johanni:87,hollerbach1983dynamic}, where an optimal velocity profile is calculated. These optimal approaches have also been extended to account for other state dependent and time dependent constraints \cite{ma2002time}, \cite{shiller1996time}.

Optimal control approaches, a version of which concerns time scaling algorithms, produce a single optimal trajectory fed to a robot to execute the desired motion. These algorithms have been successfully utilised in the last decades, however, they inevitably come with shortcomings, for example: They are  discontinuous, they are restricted to kinodynamic planning problems that specify a unique acceptable velocity for the end of the path, and offer little ability to adapt the profile along the path, as most versions generate open-loop controllers. 

To address these shortcomings, we investigate a larger class of state-feedback controllers, using the same reasoning for projecting the dynamics of the robotic manipulator on a prespecified path as in the traditional aforementioned approaches. However, instead of finding time-optimal solutions for the zero initial condition, we characterise the whole set of admissible states, namely, the distance traversed in the path and its pseudo velocity, that can reach a target set in finite time and satisfy constraints throughout. 

By projecting the dynamics on a path, any $N$-degree of freedom (DOF) manipulator is represented by a constrained double integrator. The control input corresponds to the pseudo acceleration across the path and is subject to non linear and nonconvex constraints, induced mainly by the physical constraints on the torques of the actuators. At this stage additional constraints can be imposed,  related, e.g., to potential collision energy, time constraints, or constraints relating to end effector forces that could prove useful in high precision applications, e.g., machining \cite{rossi2015pre}, \cite{olabi2010feedrate}.

Technically, rather than solving a single optimal control problem as, e.g., in \cite{Shin_McKay:85,bobrow1985time}, our main objective is to  answer the question: Can we compute the subset of states  for which a state feedback controller exists driving the system to a target set without violating any constraints in finite time?

Sets with these properties are called \emph{reach-avoid sets} \cite{8264291} \cite{landry2018reach}.  Many issues related to robots and path planning are often naturally expressed in the context of reach-avoid problems, examples include motion planning, collision avoidance, space vehicle docking, and object tracking applications \cite{fisac2015reach,shamgah2018path,homchaudhuri2016computing,kariotoglou2011stochastic}. 
To the best of our knowledge, this is the first time it is suggested to cast the kinodynamic planning problem as a reach-avoid set computation.

\begin{ilex}
 Figures \ref{fig-a} - \ref{fig-c} depict with purple color the same motion of a two link manipulator moving from point A to point B in three different spaces. The end effector work space is shown in Figure \ref{fig-a}, the space of joints is shown in Figure \ref{fig-b}. In the general case the joint space can be $N$ dimensional, depending upon the number of actuators on the robot. The state space shown in Figure \ref{fig-c} is always two dimensional, having as states the pseudo distance of the path travelled ($x_1$) and the end effector pseudo velocity  ($x_2$). We focus on designing state feedback controllers in this space (Figure \ref{fig-c}). We can incorporate and visualise additional constraints based on the system state  (configuration and speed) with blue. The target set is depicted in red in Figure \ref{fig-c}. 
\end{ilex}

Highly relevant for safety-critical systems, reach-avoid problems are not trivial. Efficient solutions depend on the nature of the system, with several techniques available for a variety of system types \cite{summers2010verification,krishna2017reach,shamgah2018symbolic}. In our case, we are exploiting the monotonicity of the projected trajectories with respect to a novel parameterisation of the control input.  This parameterisation allows ordering relations between generated closed-loop trajectories to be established in terms of both the initial conditions and the control input. These observations constitute the building blocks of an algorithmic procedure that recovers an approximation of the largest reach-avoid set. Compared to interesting related approaches \cite{6895310}, \cite{pham2018new}, our work deals directly with the nonlinear dynamics and constraints, and constructs the boundaries of the reach-avoid set using a combination of extreme trajectories. Thus, there is no added conservatism caused by the linearisation of dynamics. 

We also establish set-based state feedback controllers, induced by the parameterisation of the control strategy. Taking into account the preference for continuous, smooth, and robust state feedback control reducing wear on actuators, we establish general properties the control template must satisfy, and show how simple, Lipschitz continuous, continuous, and sliding mode-like controllers can be established. Last, we apply our framework to numerical experiments on commercial robotic manipulators, and discuss practical considerations in algorithmic implementation to show the efficacy of our approach. 
Preliminary ideas on computing reach-avoid sets for the studied problem are reported in  \cite{mcgovern2022kinodynamic}. In comparison, we formally show continuity properties of the suggested control template, establish analytic algorithms on computations of the reach-avoid set, propose new families of stabilizing feedback controllers, highlight practical challenges in implementation, and perform a formal comparison of our results with the conventional time optimal control approach.

Section \ref{section 2} discusses preliminaries. In Section \ref{section 3} we introduce the problem and the control input parameterisation, and establish the Lipshitz properties of the closed-loop system dynamics. Section \ref{section 4} establishes ordering relations between generated trajectories, while section \ref{section 5} illustrates in detail the algorithmic procedure for producing the reach-avoid set. Section \ref{section 7} deals with the characterisation and implementation of families of safe state feedback controllers. Section \ref{section 8} presents examples of the algorithms running on a Python based simulation using the model for the Universal Robots UR5 \cite{UniversalRobots} available within the Matlab file exchange \cite{Kebria_2016}. Conclusions are drawn in Section \ref{Conclusions}. For ease of exposition the proofs are placed in the Appendix.

\section{Preliminaries} \label{section 2}
\noindent \underline{Notation}:
For a vector $x\in\R^N$, $\|x \|$ is its 2-norm and $x_i$ its $i^{th}$ element.  For a vector $x_\text{label}\in\R^N$ with a label in its subscript, we write its $i^{th}$ element as $(x_\text{label})_i$. We denote sets, e.g., $\cS,\cR$, with capital letters in italics.
For a set $\cS$, $\text{conv}(\cS)$, $\boundary{\cS}$ and $\interior{(\cS)}$ denote its convex hull, boundary and interior respectively. The ball centered at zero with radius $\varepsilon$ is denoted by $\B(\varepsilon)$.  The distance of a point $x$ from a compact set $\cS$ $\text{dist}(x,\cS)=\min_{y\in\cS}\|x-y \|$. Vector inequalities hold elementwise.

\noindent
 The motion of $N$-link manipulators can be described using standard Lagrangian dynamics by $M_L(q)\ddot{q} + C_L(q, \dot{q})\dot{q} + g_L(q) = \tau,$

where $q \in \R^{N}$ corresponds to joint positions and $\tau\in\R^N$ collects the actuator torques for each joint. The matrices $M_L(q) \in \R^{N\times N}$, $C_L(q, \dot{q}) \in \R^{N\times N}$ are the inertial and centrifugal/Coriolis matrices respectively. 
The gravitational force effects are represented by $g_L(q)$. The matrix and vector functions $M_L(q)$, $C_L(q)$ and $g_L(q)$ are globally Lipschitz, \cite{choi2001iterative,bouakrif2013velocity,de2001commande}.

The Lagrangian dynamics can be projected to path dynamics: The first step is to consider an end effector path described by a function $q(s):[0,1]\rightarrow \R^n$ known as a path parameterisation where $s\in[0,1]$ is a scalar representing each configuration of the robot as it moves along the path.

Details can be found in the seminal papers \cite{Shin_McKay:85,bobrow1985time}  and the  book \cite[Chapter 9]{lynch_park}.

\begin{assumption} \label{path_assumption}
The path parameterisation $q(s)$ is two times differentiable. 
Moreover, 
$\left\|\frac{dq(s)}{ds}\right\|\geq \alpha_0$,
for $s\in [0,1]$ and a constant positive number $\alpha_0 > 0$.
\end{assumption}

Assumption \ref{path_assumption} is not restrictive and is common in the literature \cite{Shin_McKay:85,bobrow1985time}, \cite{lynch_park}. Intuitively, it requires that the configuration must be able to change such that the end effector can traverse a path in finite time.

The parameterised system becomes $M(s)\ddot{s} + C(s)\dot{s}^2 + g(s) = \tau$ where $M(s) \in \R^{N}$, $C(s) \in \R^{N}$, $g(s) \in \R^{N}$, are the transformed vectors, and  $\tau \in \R^{N}$ describes the torques and forces produced by the actuators. We define 
$x = \begin{bmatrix} s & \dot{s} \end{bmatrix}^T =  \begin{bmatrix} x_1 & x_2\end{bmatrix}^\top$. The torques are limited by the actuator dynamics as $\tau^{\min}_i(x)  \leq \tau_{i}(x) \leq \tau^{\max}_i(x)$ where $\tau^{\min}_i(x)$ and $\tau^{\max}_i(x)$ describe the maximum torque that can be generated in either direction.
The $i^{th}$ element, $i=1,...,N$, of the vectors within the parameterised system generates constraints of the form
\begin{equation} \label{eq:path_dynamics ith row}
\tau^{\min}_i(x) \leq M_i(x_1)\dot{x}_2 + C_i(x_1)x_2^2 + g_i(x_1) \leq \tau^{\max}_i(x).
\end{equation}

\begin{assumption} \label{ass_actuator_lims}
The state-dependent constraint bounds $\tau^{\min}_i(x)$ and $\tau^{\max}_i(x)$ are locally Lipschitz in x for the whole range of admissible motions, for all $i=1,...,N$.
\end{assumption}

Assumption \ref{ass_actuator_lims} is not restrictive for electrical motors due to their continuous relationships between torque and speed, see, e.g., \cite{Hughes_Electrical} \cite[Chapter~10]{spong_robotics}.

Rearranging \eqref{eq:path_dynamics ith row}, the limits on possible accelerations $\dot{x}_2=\ddot{s}$ from any given state $x$ can be formed. We set $    A_i(x)  = \frac{\tau^{\min}_i(x) - C_i(x_1)x_2^2 - g_i(x_1)}{M_i(x_1)}$, 
$D_i(x)  = \frac{\tau^{\max}_i(x) - C_i(x_1) x_2^2 - g_i(x_1)}{M_i(x_1)}$ for $M_i(x_1)<0$, and
$A_i(x)  = \frac{\tau^{\max}_i(x) - C_i(x_1)x_2^2 - g_i(x_1)}{M_i(x_1)}$, $D_i(x)  = \frac{\tau^{\min}_i(x) - C_i(x_1)x_2^2 - g_i(x_1)}{M_i(x_1)}$  for $M_i(x_1)>0$.
The case of zero inertia points, i.e., when $M_i(x_1)=0$, for one or more indices $i=1,..,N$ is taken into account by posing algebraic constraints on $x$  \cite{lynch_park} of the form 
$\tau^{\min}_i(x)\leq C_i(x_1)x_2^2 + g_i(x_1) \leq \tau^{\max}_i(x).$ 
By grouping the above constraints,  we can write 
\begin{equation}\label{eq_constraints}
    D(x)  \leq \ddot{s} \leq A(x) 
\end{equation}
where
\begin{align}
    A(x) &= \min_{\{i\in[1,N]:M_i(x)\neq 0\}} A_i(x) ,\label{eq:minA}\\
    {D}(x) &= \max_{\{i\in[1,N]:M_i(x)\neq 0\}}{D_i(x)}.\label{eq:maxD}
\end{align}
The projected system is a double integrator
\begin{equation}\label{eq_system_state_equations}
    \dot{x}=\Phi x+Eu(x),
\end{equation}
with $\Phi=\begin{bmatrix} 0 & 1\\ 0 & 0\end{bmatrix}$, $E=\begin{bmatrix} 0 \\ 1\end{bmatrix}$. The state feedback $u(\cdot):\R^2\rightarrow \R$ accounts for the acceleration $\ddot{s}$, whose admissible range is state-dependent \eqref{eq_constraints}. The solution of \eqref{eq_system_state_equations} at time $t$ starting from an initial condition $x_0$, under the control function $u(x)$ is 
\begin{equation}\label{eq:system at time t}
    \phi_\text{f}(t;x_0,u(x))=e^{\Phi t} x_{0}  + \int_{0}^{t} e^{ \Phi (t-\tau)}E u(x) d \tau. 
\end{equation}
The state and input constraints are $x\in \cX', \ \  u(x) \in\cU_x,$ where $\cX'   =  \cC \cap \cC_0$, with  $\cC  = \left\{x\in\R^2: B(x) \leq 0 \right\}$, $    B(x) = \begin{bmatrix}-x_1 & x_1-1 & -x_2 & D(x) - A(x) \end{bmatrix}^\top$. Moreover,  $\cC_0  = \{ x\in\R^2 : (\exists j \in [1,N]: M_j(x_1) = 0, \quad 
\tau^{min}_{j} - C_{j}(x)x_2^2 - g_{j}(x) \leq 0, \quad  -\tau^{max}_{j} + C_{j}(x)x_2^2 + g_{j}(x) \leq 0  ) \}$, and $\cU_x = \{ u(x)\in\R: D(x)\leq u(x)\leq A(x) \}$. We note that within our framework we can incorporate additional constraints. We assume we can eventually describe all  constraints via two functions $C^{u}(x_1):[0,1]\to\R$ and $C^{l}(x_1):[0,1]\to\R$,  representing upper and lower bounds respectively.
This leads to the description of the constraint set \begin{equation}\label{eq: admisible region additional constraint set}
\cX = \left\{ x\in\R^2 : C^{l}(x_1) \leq x_2 \leq C^{u}(x_1), x_1\in[0,1] \right\}.
\end{equation}
The overall state (we note $\cX\subseteq \cX')$ and input constraints are 
\begin{equation}
 x\in \cX, \ \  u(x) \in\cU_x. \label{eq:admissible and input constraints}
\end{equation}

\begin{assumption} \label{assumption_3}
The functions $x_2=C^{u}(x_1) $ and $x_2=C^{l}(x_1) = x_2$ are semi-differentiable continuous bijections.
\end{assumption}

We note Assumption~\ref{assumption_3} is indeed somewhat restrictive, as, for example, it does not allow islanding. Nevertheless, it covers many realistic and challenging cases, including polynomial and trigonometric constraints. 

\section{Control parameterisation and  properties}\label{section 3}

In this section, we propose a  parameterisation of the state feedback controller as a convex combination of the state dependent input constraints.
The generally non-restrictive Assumptions \ref{path_assumption}, \ref{ass_actuator_lims}  allow us to characterize the parameterisation with Lipschitz continuity, which constitutes our first main technical result.  Preliminary results and the proof of Proposition~\ref{proposition 1}  can be found in the Appendix.

\begin{proposition}\label{proposition 1}
The functions $A(x)$, $D(x)$ are locally Lipschitz continuous in $\cX'$.
\end{proposition}

In agreement with approaches  in the literature \cite{Shin_McKay:85}, \cite{lynch_park}, we set $u(x) =\dot{x_2}$ as the input variable. 
We build all subsequent results in the following parameterisation of the control law via the \emph{actuation level function} $\lambda(x): \cX \rightarrow [0,1]$ as follows
\begin{equation}\label{eq: input definition}
u(x, \lambda(x)) = D(x) + \lambda(x) (A(x) - D(x)).
\end{equation}
The following result is a consequence of Proposition \ref{proposition 1}.
\begin{corollary}\label{col: lambda x}
Consider the system  \eqref{eq_system_state_equations}  constraints \eqref{eq:admissible and input constraints}, and a Lipschitz continuous function $\lambda(x):\cX\rightarrow [0,1]$. Then \eqref{eq: input definition} is locally Lipschitz continuous in $\cX$.
\end{corollary}

The forward and backward dynamics of the system are 
\begin{align}
\dot{x} & =  \Phi x + E u(x, \lambda(x))\label{eq:state dynamics forward},\\
\dot{x} & = -\Phi x  -E u(x, \lambda(x))\label{eq:state dynamics reverse}
\end{align}
respectively, with $\Phi$, $E$ defined in \eqref{eq_system_state_equations}.

We consider target sets of the form
\begin{equation}\label{target set}
    \cX_{\text{T}}=\{x \in \cX: x_1=c, \ \  x^\ast_l\leq  x_2\leq x^\ast_u  \},
\end{equation}
where $0<c\leq 1$ and  $x^\ast_u,  x^\ast_l \in \R$.
Given a particular choice of the actuation level function $\lambda(x)$ \eqref{eq: input definition} and an initial condition $x_0\in\cX$, the solution of \eqref{eq:state dynamics forward} is $\phi_\text{f}(t;x_0,\lambda(x))=e^{\Phi t} x_{0}  + \int_{0}^{t} e^{\Phi (t-\tau)}E u(x, \lambda(x)) d \tau $ where $e^{\Phi t}=\begin{bmatrix} 1 & t \\ 0 & 1\end{bmatrix}$.
Correspondingly, the solution of the backwards dynamics \eqref{eq:state dynamics reverse} is $\phi_\text{b}(t;x_0,\lambda(x)))=e^{-\Phi t} x_{0} - \int_{0}^{t} e^{-\Phi (t-\tau)}E u(x, \lambda(x)) d \tau$. 

\begin{definition}
A set $\cY\subseteq \cX$ is a \emph{reach-avoid set} with respect to the system \eqref{eq:state dynamics forward},  constraints \eqref{eq:admissible and input constraints} and a target set $\cX_\text{T}$ \eqref{target set}
if  any initial condition $x_0\in\cY$  can be transferred to  $\cX_T$ in finite time under an admissible control law $u(x)$ \eqref{eq: input definition}, i.e., there is $t^\star>0$ such that  $\phi_\text{f}(t^*;x_0,u(x))\in\cX_\text{T}$, and $\phi_\text{f}(t;x_0,u(x))\in\cX$ for all $t \in [0,t^*]$.
\end{definition}
\begin{definition} 
Consider the system  \eqref{eq:state dynamics forward}, the constraint set $\cX$ \eqref{eq: admisible region additional constraint set}, a target set $\cX_\text{T}$ \eqref{target set} and the controller $u(x)$ given by \eqref{eq: input definition}. The \emph{maximal reach-avoid set} is
\begin{align} \label{eq: reach-avoid set}
\cR(\cX_{\text{T}})  = & \{  x_0 \in \cX: ( \exists t^{*}>0: \phi_{f}(t^*,x_0,u(x))\in\cX_{\text{T}},  \nonumber \\
& \;\;\;\;  \phi_{f}(t,x_0,u(x)) \in \cX, \forall t\in[0,t^*]  ) \}.
\end{align}
\end{definition}

\section{Ordering properties of closed-loop trajectories}\label{section 4}

We collect the admissible trajectories of the closed loop system in the forward and backwards direction in the sets
\begin{equation} \label{trajectory forwards}
    \cT_\text{f}(x_0,\lambda(x))=\{x \in \cX: (\exists t\geq 0: x=\phi_{\text{f}}(t;x_0,\lambda(x)))\},
\end{equation}
\begin{equation} \label{trajectory backwards}
    \cT_\text{b}(x_0,\lambda(x))=\{x \in \cX: (\exists t\geq 0: x=\phi_\text{b}(t;x_0,\lambda(x)))\},
\end{equation}
respectively. In the remainder, and to discuss trajectories without needing to specify the direction in which they are formed, we write $\cT(x_0, \lambda(x))$. We define intervals $\cI= [(x_a)_1, (x_b)_1]$  $(x_a)_1\leq (x_b)_1\leq 1$. The \emph{slice} of a set $\cX$ on an interval $\cI$ is
\begin{equation}
\cW(\cX,\cI) = \cX \cap \left\{ x: x_1\in \cI \right\}.\label{slice description}
\end{equation}

The mapping \eqref{slice description} is valid also when $\cI=\{x_a\}$ is a singleton. In this case, we write for simplicity $\cW(\cX,x_a)$. 

\begin{lemma}\label{rule 1}
Let $n\in[0,1]$. Consider the slice $\cW(\cX, n)$, two states $x_{u}, x_{l} \in \cW(\cX, n)$ such that $(x_u)_2> (x_l)_2$, and a constant actuation level $\lambda(x)=\lambda \in [0,1]$. Then,
$$\cT(x_{u}, \lambda) \cap \cT(x_{l}, \lambda) = \emptyset.$$
\end{lemma}
\begin{lemma}\label{rule 2}
Let $n\in[0,1]$. Consider two states $x_{u}, x_{l} \in \cW(\cX, n)$ such that $ x_{u} \neq x_{l}$. Consider two fixed actuation levels   $0 \leq \lambda_l < \lambda_u\leq 1 $. Then,  the trajectories can  intersect at most once in $\interior{(\cX)}$, i.e., either 
$\cT(x_{u}, \lambda_u) \cap \cT(x_{l}, \lambda_l)=\{x^\star \}$, for some $x^\star \in \interior{\cX}$, or $\cT(x_{u}, \lambda_u) \cap \cT(x_{l}, \lambda_l)=\emptyset$.
\end{lemma}

A straightforward consequence of Lemma \ref{rule 2} is that for any two trajectories evolving from a single state with different constant actuation levels the one with the higher actuation level will remain above the other for the duration of the motion. Specifically, for any $x_1^\star\in[0,1]$, any $\lambda_b\leq \lambda_a\leq 1$ it holds that $x_b\leq x_a$, where $x_a= \cW(\cX, x_1^\ast) \cap\mathcal{T}_{f}(x_0,\lambda_a)$, $x_b=\cW(\cX, x_1^\ast) \cap  \mathcal{T}_{f}(x_0,\lambda_b)$. This is summarised in the following corollary.

\begin{corollary}
Consider the system  \eqref{eq:state dynamics forward}, the constraint set $\cX$ \eqref{eq: admisible region additional constraint set}, trajectories \eqref{trajectory forwards} and \eqref{trajectory backwards} with constant actuation levels $\lambda_a, \lambda_b$. 
The following hold.
\begin{equation*}
    \cW(\cX, n)\cap \cT_{f}(x_0,\lambda_a)\geq \cW(\cX, n)\cap \cT_{f}(x_0,\lambda_b),
\end{equation*}
\begin{equation*}
    \cW(\cX, n)\cap \cT_{b}(x_0,\lambda_b)\geq \cW(\cX, n)\cap \cT_{b}(x_0,\lambda_a),
\end{equation*}
for any $0<\lambda_b\leq \lambda_a\leq 1$, for all $x_0\in\cX$, and $n \in [0,1]$.
\end{corollary}

Moreover, Lemma \ref{rule 2} implies that in the case where $(x_{u})_{2} > (x_{l})_{2}$, we have $\cT_1 \cap \cT_2 = \emptyset$ for any two trajectories $\cT_1 = \cT(x_u, \lambda_u), \cT_2 = \cT(x_l, \lambda_l)$, where $x_u$, $x_l$ lie on a slice $\cW(\cX, x_{1})$ where $x_1\in[0,1]$ such that $x_{u} \neq x_{l}$ and $\lambda_u>\lambda_l$.

\begin{corollary}
Consider the system  \eqref{eq:state dynamics forward}, the constraint set $\cX$ \eqref{eq: admisible region additional constraint set}, trajectories \eqref{trajectory forwards} with constant actuation levels $\lambda_u, \lambda_l$ and $n\in[0,1]$. Then, for any $x_u, x_l$ such that $(x_u)_1=(x_l)_1$, $(x_u)_2>(x_l)_2$,  any $0\leq \lambda_l<\lambda_u\leq 1$ it holds
\begin{equation*}
    \cW(\cX, n)\cap \cT_f(x_u,\lambda_u)>\cW(\cX, n)\cap \cT_f(x_l,\lambda_l).
\end{equation*}
\end{corollary}

\section{Approximation of the maximal Reach-avoid set} \label{section 5}
 
We first define the hypograph and epigraph of a set $\cZ\subseteq \cX$  respectively
\begin{align*}
\cR_u(\cZ) & = \{x\in \cX: (\exists y\in\cZ: x_1=y_1, x_2\leq y_2 \}, \\
\cR_l(\cZ) & = \{x\in \cX: (\exists y\in\cZ: x_1=y_1, x_2\geq y_2 \}.
\end{align*}
We also define the set-valued operation $\cQ(\cdot)$ that returns the `leftmost' state  of a trajectory $\cT=\cT(x_0,\lambda(x))$ in $\cX$    
\begin{equation*}
    \cQ(\cT)=\{x^\ast\in\cX\cap\cT: x_1^\ast\leq x_1, \forall x\in\cX\cap\cT \}.
\end{equation*}
The sets representing  the upper and lower boundaries of the $\cX$ \eqref{eq: admisible region additional constraint set} are defined below
\begin{align}
    \cV^u &= \left\{ x\in\cX:  C^{u}(x_1) = x_2  \right\}, \label{upper ACC}\\
    \cV^l &= \left\{ x\in\cX:  C^{l}(x_1) = x_2 \right\}. \label{lower ACC}
\end{align}

\begin{algorithm}[h]\label{algorithm 1}
\caption{Reach-avoid set computation}
\begin{algorithmic}[1]
\State $\mathbf{Input:}$ $\cX$, $\cX_{\text{T}}$, $\cV^l$, $\cV^u$,  $\epsilon>0$
\State $\mathbf{Output:}$  $\cR_{\epsilon}(\cX_\text{T})$ 
\State $\cT_u^{*} \leftarrow \cT_\text{b}(x^*_u, 0)$,   
\State $\{x_{d}\} \leftarrow \cQ( \cT_u^{*} )$ 
\State $\cT_l^{*} \leftarrow \cT_\text{b}(x^*_l, 1)$
\State $\{x_{a} \}\leftarrow \cQ( \cT_l^{*} )$
\If {$x_{a}\in \cV^l \AND x_{d}\in \cV^l$}
    \State $\cZ_l \leftarrow \cT_l^{*} \cup  \text{extend}(\cV^l, [(x_d)_1, (x_a)_1], 1, \epsilon)$
    \State $\cZ_u \leftarrow \cT_u^{*}$
\ElsIf{$x_{a} \in \cV^u \AND x_{d} \in \cV^u$ } 
    \State $\cZ_{u} \leftarrow \cT_u^{*} \cup \text{extend}(\cV^u, [(x_d)_1, (x_a)_1], 0, \epsilon)$
    \State $\cZ_l \leftarrow \cT_l^{*}$
\Else
    \If{$x_{a} \in \cV^l$}
        \State $\cZ_{l} \leftarrow \cT_{l}^{*} \cup \text{extend}(\cV^l, [0, (x_a)_1], 1, \epsilon)$
    \Else
        \State $Z_l \leftarrow \cT_l^{*}$ 
    \EndIf
    \If{$x_{d} \in \cV^u$}
    \State $\cZ_{u} \leftarrow \cT_u^* \cup \text{extend}(\cV^u, [0, (x_d)_1], 0, \epsilon)$
    \Else
        \State $\cZ_u \leftarrow \cT_u^*$
    \EndIf
\EndIf
\State \textbf{return} $\cR_\epsilon (\cX_\text{T}) \leftarrow \cR_{l}(\cZ_u) \cap \cR_{u}(\cZ_l)$
\end{algorithmic}
\end{algorithm}
\begin{figure}[h]
    \begin{subfigure}{0.23\textwidth}
        \centering
        \begin{tikzpicture}[thick,scale=0.5, every node/.style={scale=0.5}]
        \draw[->,thin] (0,0)--(7,0) node[right]{$x_1$};
        \draw[->,thin] (0,0)--(0,5) node[above]{$x_2$};
        \draw[constraint_color, thick] (0.05, 3) .. controls (1.5, 4) and (2.25, 1) .. (3,4) node[above, xshift=-0.2cm]{$\cV^{u}$};
        \draw[constraint_color, thick] (3, 4) .. controls (4, 6) and (5, 3) .. (6,5);
        \draw[constraint_color, thick] (6, 5) -- (6, 0.05);
        \draw[constraint_color, thick] (0.05, 0.05) .. controls (3, 2) and (3.5, 0.05) .. (6, 0.05)node[above, xshift=-4cm, yshift=0.25cm]{$\cV^{l}$};
        \draw[constraint_color, thick] (0.05, 0.05) -- (0.05, 3);
        \draw[red!90!black, thick] (6, 2) -- (6, 3)node[right, yshift=-0.5cm, xshift=0cm]{$\cX_\text{T}$};
        \fill[black,thick,dashed] (6,2) circle (0.05cm) node[right, yshift=-0.25cm, xshift=0cm]{$x^{*}_l$};
        \fill[black,thick,dashed] (6,3) circle (0.05cm) node[right, yshift=0.25cm, xshift=0cm]{$x^{*}_u$};
        \draw[magenta!70!black, thick] (6, 3) .. controls (5, 4) and (4, 3) .. (3,0.85) node[above, yshift=2.2cm, xshift=1.25cm]{$T^{*}_u$};
        \fill[black,thick,dashed] (3,0.85) circle (0.05cm) node[below, yshift=0cm, xshift=0cm]{$x_d$};
        \draw[magenta!70!black, thick] (6, 2) .. controls (5.75, 2) and (5.25, 2) .. (5,0.15) node[above, yshift=0.5cm, xshift=0.5cm]{$T^{*}_l$};
        \fill[black,thick,dashed] (5,0.15) circle (0.05cm) node[right, yshift=0.1cm, xshift=0cm]{$x_a$};
        \end{tikzpicture}
        \subcaption{Situation 1} \label{alg1 a}
    \end{subfigure}
    \begin{subfigure}{0.23\textwidth}
        \centering
        \begin{tikzpicture}[thick,scale=0.5, every node/.style={scale=0.5}]
        \draw[->,thin] (0,0)--(7,0) node[right]{$x_1$};
        \draw[->,thin] (0,0)--(0,5) node[above]{$x_2$};
        \draw[constraint_color, thick] (0.05, 3) .. controls (1.5, 4) and (2.25, 1) .. (3,4) node[above, yshift=0.8cm, xshift=0.5cm]{$\cV^{u}$};
        \draw[constraint_color, thick] (3, 4) .. controls (4, 6) and (5, 3) .. (6,5);
        \draw[constraint_color, thick] (6, 5) -- (6, 0.05);
        \draw[constraint_color, thick] (0.05, 0.05) .. controls (3, 2) and (3.5, 0.05) .. (6, 0.05)node[above, xshift=-4cm, yshift=0.25cm]{$\cV^{l}$};
        \draw[constraint_color, thick] (0.05, 0.05) -- (0.05, 3);
        \draw[red!90!black, thick] (6, 2) -- (6, 3)node[right, yshift=-0.5cm, xshift=0cm]{$\cX_\text{T}$};
        \fill[black,thick,dashed] (6,2) circle (0.05cm) node[right, yshift=-0.25cm, xshift=0cm]{$x^{*}_l$};
        \fill[black,thick,dashed] (6,3) circle (0.05cm) node[right, yshift=0.25cm, xshift=0cm]{$x^{*}_u$};
        \draw[magenta!70!black, thick] (6, 3) .. controls (5, 4) and (4, 3) .. (3,4) node[above, yshift=-0.50cm, xshift=2cm]{$T^{*}_u$};
        \fill[black,thick,dashed] (3,4) circle (0.05cm) node[left, yshift=0cm, xshift=0cm]{$x_d$};
        \draw[magenta!70!black, thick] (6, 2) .. controls (4, 2) and (3, 2) .. (2,2.65) node[above, yshift=-1.3cm, xshift=3cm]{$T^{*}_l$};
        \fill[black,thick,dashed] (2,2.65) circle (0.05cm) node[below, yshift=0cm, xshift=0cm]{$x_a$};
        \end{tikzpicture}
        \subcaption{Situation 2} \label{alg1 b}
    \end{subfigure}
    \begin{subfigure}{0.23\textwidth}
        \centering
        \begin{tikzpicture}[thick,scale=0.5, every node/.style={scale=0.5}]
        \draw[->,thin] (0,0)--(7,0) node[right]{$x_1$};
        \draw[->,thin] (0,0)--(0,5) node[above]{$x_2$};
        \draw[constraint_color, thick] (0.05, 3) .. controls (1.5, 4) and (2.25, 1) .. (3,4) node[above, yshift=0.8cm, xshift=0.5cm]{$\cV^{u}$};
        \draw[constraint_color, thick] (3, 4) .. controls (4, 6) and (5, 3) .. (6,5);
        \draw[constraint_color, thick] (6, 5) -- (6, 0.05);
        \draw[constraint_color, thick] (0.05, 0.05) .. controls (3, 2) and (3.5, 0.05) .. (6, 0.05)node[above, xshift=-4cm, yshift=0.25cm]{$\cV^{l}$};
        \draw[constraint_color, thick] (0.05, 0.05) -- (0.05, 3);
        \draw[red!90!black, thick] (6, 2) -- (6, 3)node[right, yshift=-0.5cm, xshift=0cm]{$\cX_\text{T}$};
        \fill[black,thick,dashed] (6,2) circle (0.05cm) node[right, yshift=-0.25cm, xshift=0cm]{$x^{*}_l$};
        \fill[black,thick,dashed] (6,3) circle (0.05cm) node[right, yshift=0.25cm, xshift=0cm]{$x^{*}_u$};
        \draw[magenta!70!black, thick] (6, 3) .. controls (5, 4) and (4, 3) .. (3,4) node[above, yshift=-0.50cm, xshift=2cm]{$T^{*}_u$};
        \fill[black,thick,dashed] (3,4) circle (0.05cm) node[left, yshift=0cm, xshift=0cm]{$x_d$};
        \draw[magenta!70!black, thick] (6, 2) .. controls (4, 2) and (3, 2) .. (0.05,1) node[above, yshift=0.3cm, xshift=5.25cm]{$T^{*}_l$};
        \fill[black,thick,dashed] (0.05,1) circle (0.05cm) node[left, yshift=0cm, xshift=0cm]{$x_a$};
        \end{tikzpicture}
        \subcaption{Situation 3} \label{alg1 c}
    \end{subfigure}
    \begin{subfigure}{0.23\textwidth}
        \centering
        \begin{tikzpicture}[thick,scale=0.5, every node/.style={scale=0.5}]
        \draw[->,thin] (0,0)--(7,0) node[right]{$x_1$};
        \draw[->,thin] (0,0)--(0,5) node[above]{$x_2$};
        \draw[constraint_color, thick] (0.05, 3) .. controls (1.5, 4) and (2.25, 1) .. (3,4) node[above, xshift=-0.2cm]{$\cV^{u}$};
        \draw[constraint_color, thick] (3, 4) .. controls (4, 6) and (5, 3) .. (6,5);
        \draw[constraint_color, thick] (6, 5) -- (6, 0.05);
        \draw[constraint_color, thick] (0.05, 0.05) .. controls (3, 2) and (3.5, 0.05) .. (6, 0.05)node[above, xshift=-4cm, yshift=0.25cm]{$\cV^{l}$};
        \draw[constraint_color, thick] (0.05, 0.05) -- (0.05, 3);
        \draw[red!90!black, thick] (6, 2) -- (6, 3)node[right, yshift=-0.5cm, xshift=0cm]{$\cX_\text{T}$};
        \fill[black,thick,dashed] (6,2) circle (0.05cm) node[right, yshift=-0.25cm, xshift=0cm]{$x^{*}_l$};
        \fill[black,thick,dashed] (6,3) circle (0.05cm) node[right, yshift=0.25cm, xshift=0cm]{$x^{*}_u$};
        \draw[magenta!70!black, thick] (6, 3) .. controls (5, 4) and (2, 1) .. (0.05,2) node[above, yshift=1.3cm, xshift=5cm]{$T^{*}_u$};
        \fill[black,thick,dashed] (0.05,2) circle (0.05cm) node[left, yshift=0cm, xshift=0cm]{$x_d$};
        \draw[magenta!70!black, thick] (6, 2) .. controls (2, 0) and (3, 2) .. (0.05,0.5) node[above, yshift=0.3cm, xshift=5cm]{$T^{*}_l$};
        \fill[black,thick,dashed] (0.05,0.5) circle (0.05cm) node[left, yshift=0cm, xshift=0cm]{$x_a$};
        \end{tikzpicture}
        \subcaption{Situation 4} \label{alg1 d}
    \end{subfigure}
    \begin{subfigure}{0.23\textwidth}
        \centering
        \begin{tikzpicture}[thick,scale=0.5, every node/.style={scale=0.5}]
        \draw[->,thin] (0,0)--(7,0) node[right]{$x_1$};
        \draw[->,thin] (0,0)--(0,5) node[above]{$x_2$};
        \draw[constraint_color, thick] (0.05, 3) .. controls (1.5, 4) and (2.25, 1) .. (3,4) node[above, xshift=-0.2cm]{$\cV^{u}$};
        \draw[constraint_color, thick] (3, 4) .. controls (4, 6) and (5, 3) .. (6,5);
        \draw[constraint_color, thick] (6, 5) -- (6, 0.05);
        \draw[constraint_color, thick] (0.05, 0.05) .. controls (3, 2) and (3.5, 0.05) .. (6, 0.05)node[above, xshift=-4cm, yshift=0.25cm]{$\cV^{l}$};
        \draw[constraint_color, thick] (0.05, 0.05) -- (0.05, 3);
        \draw[red!90!black, thick] (6, 2) -- (6, 3)node[right, yshift=-0.5cm, xshift=0cm]{$\cX_\text{T}$};
        \fill[black,thick,dashed] (6,2) circle (0.05cm) node[right, yshift=-0.25cm, xshift=0cm]{$x^{*}_l$};
        \fill[black,thick,dashed] (6,3) circle (0.05cm) node[right, yshift=0.25cm, xshift=0cm]{$x^{*}_u$};
        \draw[magenta!70!black, thick] (6, 3) .. controls (5, 4) and (2, 1) .. (0.05,2) node[above, yshift=1cm, xshift=4.5cm]{$T^{*}_u$};
        \fill[black,thick,dashed] (0.05,2) circle (0.05cm) node[left, yshift=0cm, xshift=0cm]{$x_d$};
        \draw[magenta!70!black, thick] (6, 2) .. controls (5, 1) and (3, 2) .. (2,0.9) node[above, yshift=0.7cm, xshift=3cm]{$T^{*}_l$};
        \fill[black,thick,dashed] (2,0.9) circle (0.05cm) node[left, yshift=0.3cm, xshift=0cm]{$x_a$};
        \draw [draw=gray, dashed] (-0.1,-0.1) rectangle (2.5,1.5);
        \end{tikzpicture}
        \subcaption{Situation 5} \label{alg1 e}
    \end{subfigure}
    \begin{subfigure}{0.23\textwidth}
        \centering
        \begin{tikzpicture}[thick,scale=0.5, every node/.style={scale=0.5}]
        \draw[->,thin] (0,0)--(7,0) node[right]{$x_1$};
        \draw[->,thin] (0,0)--(0,5) node[above]{$x_2$};
        \draw[constraint_color, thick] (0.05, 3) .. controls (1.5, 4) and (2.25, 1) .. (3,4) node[above, xshift=-0.2cm]{$\cV^{u}$};
        \draw[constraint_color, thick] (3, 4) .. controls (4, 6) and (5, 3) .. (6,5);
        \draw[constraint_color, thick] (6, 5) -- (6, 0.05);
        \draw[constraint_color, thick] (0.05, 0.05) .. controls (3, 2) and (3.5, 0.05) .. (6, 0.05)node[above, xshift=-4cm, yshift=0.25cm]{$\cV^{l}$};
        \draw[constraint_color, thick] (0.05, 0.05) -- (0.05, 3);
        \draw[red!90!black, thick] (6, 2) -- (6, 3)node[right, yshift=-0.5cm, xshift=0cm]{$\cX_\text{T}$};
        \fill[black,thick,dashed] (6,2) circle (0.05cm) node[right, yshift=-0.25cm, xshift=0cm]{$x^{*}_l$};
        \fill[black,thick,dashed] (6,3) circle (0.05cm) node[right, yshift=0.25cm, xshift=0cm]{$x^{*}_u$};
        \draw[magenta!70!black, thick] (6, 3) .. controls (5, 4) and (2, 1) .. (1.5,2.9) node[above, yshift=0.2cm, xshift=3cm]{$T^{*}_u$};
        \fill[black,thick,dashed] (1.5,2.9) circle (0.05cm) node[above, yshift=0.3cm, xshift=0cm]{$x_d$};
        \draw[magenta!70!black, thick] (6, 2) .. controls (5, 1) and (3, 2) .. (2,0.9) node[above, yshift=0.7cm, xshift=3cm]{$T^{*}_l$};
        \fill[black,thick,dashed] (2,0.9) circle (0.05cm) node[left, yshift=0.3cm, xshift=0cm]{$x_d$};
        \end{tikzpicture}
        \subcaption{Situation 6} \label{alg1 f}
    \end{subfigure}
\caption{Illustration of the possible configurations that can result from Lines 3-5 from Algorithm 1. Figure \ref{fig: intervals} details the area indicated in subfigure (e).} \label{prototype algorithm 1 illustration}
\end{figure}
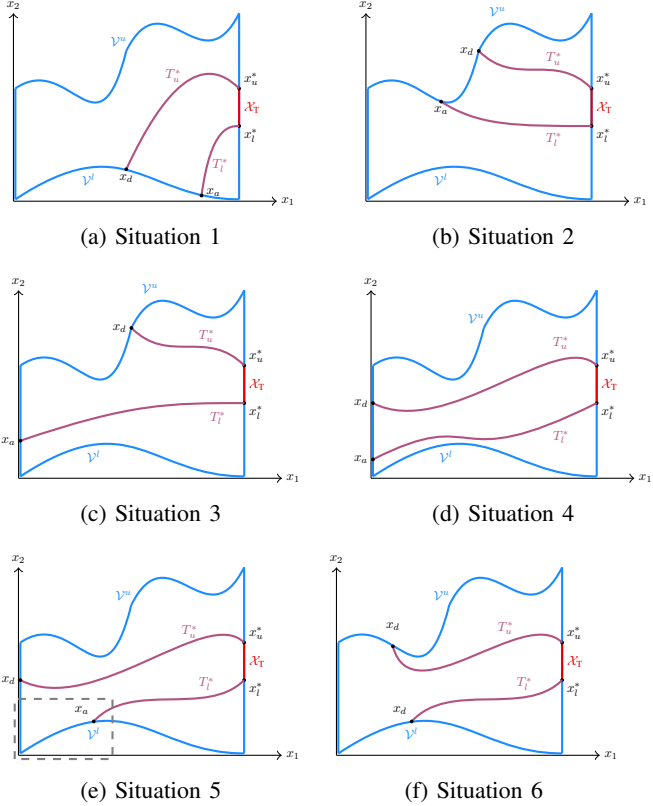
The results of Section \ref{section 4} lay the foundation for the development of the reach avoid set computation, outlined in Algrorithm 1.
Initially, the two trajectories $T_l^*$ and $T_u^*$,  produced from backward integration of the extreme points of $\cX_\text{T}$, are computed in Lines 3 and 5 respectively. We denote with $x_a$ and $x_d$ the intersections of these trajectories with the constraint set  (Lines 4, 6), with $x^*_u$, $x^*_l$ defining the interval $\cX_{\text{T}}$  \eqref{target set}.
Lines 7, l0 and 13 represent the different  cases that can occur, depending on where $x_d$ and $x_a$ lie. Figure \ref{prototype algorithm 1 illustration} illustrates all  possible situations: Line 7 covers Figure \ref{alg1 a}, where $x_d,x_a$ are in $\cV^l$. In this case no additional computations are needed as $\cZ_u =\cT^{*}_u$. On the other hand, the lower bound needs to be extended to obtain $\cZ_l$. 
Line 10 covers Figure \ref{alg1 b}. In this case, the lower bound is complete,  however, the upper bound possibly needs to be modified.
Lines 13-21 cover the remaining cases shown in Figures \ref{alg1 c} - \ref{alg1 f}, that occur when $\cT^*_u$ and $\cT^*_l$ do not intersect the same boundary $\cV^u$ or $\cV^l$. For these instances, the \texttt{extend} operation is applied,  which is described in detail below in Algorithm 2.
After all the relevant extensions have been made Line 22  returns $\cR_\epsilon (\cX_\text{T})$ which is constructed as the intersection of the sets $\cR_l(\cZ_u)$, and $\cR_u(\cZ_l)$. 

We note that the target set $\cX_\text{T}$ is the only \emph{admissible exit facet}, i.e., the area of the boundary where the trajectory $\cT_f(x_0, \lambda(x))$ will escape $\cR_\epsilon(\cX_T)$ in finite time  under appropriate control $\lambda(x)$. Several existing works on \emph{control-to-facet} strategies focus on non-linear hybrid systems with sets defined by simplices or other simple polynomial descriptions of the set boundaries \cite{habets2006reachability,habets2006control,sloth2014control}. 

\subsection{Extension of bounds of $\cR_{\epsilon}(\cX_\text{T})$}

The operation \texttt{extend}($\cV$, $[(x_{\text{end}})_1, (x_{\text{start}})_1]$, $\lambda$, $\epsilon$) is utilised in Algorithm 1 on Lines 8, 11, 15, and 19. We use $\cV$ to denote either the lower or the upper boundary, i.e., $\cV \in \{\cV^u, \cV^l\}$ and  $\cV^u, \cV^l$ are defined by \eqref{upper ACC} \eqref{lower ACC}.  
 The operation returns a curve that constrains any trajectory $\cT(x_0, \lambda(x))$, with $x_0 \in \cR_\epsilon (\cX_\text{T})$, when $\lambda(x)$ is suitably chosen.
The set $\cZ$ is effectively formed as a union of extreme trajectories and segments of $\cV$ by applying the backward dynamics \eqref{eq:state dynamics reverse}. To describe the proposed procedure we define Bouligand's tangent cone, see, e.g.,  \cite[p. 122]{Blanchini}.
\begin{definition} \label{tangent cone}
The tangent cone  of a vector $x\in\cX$  to $\cX$ \eqref{eq: admisible region additional constraint set} is
\begin{equation}
\cB(x) = \left\{ z: \liminf \limits_{\tau \to 0} \frac{\text{dist}(x+\tau z, \cX)}{\tau} = 0 \right \}. \label{tangent cone}
\end{equation}
\end{definition}

\begin{definition}\label{def: input dynamics cone}
 Consider a state $x \in \cV$ and the system dynamics \eqref{eq:state dynamics forward} $\dot{x} = f(x, \lambda)= \begin{bmatrix} x_2 & u(x, \lambda) \end{bmatrix}^\top$, with $\lambda \in [0,1]$. We define the set
 \begin{align}
\cF(x) = \{f(x, \lambda): \lambda \in [0,1] \}. \label{velocity cone}
\end{align}

\end{definition}
Utilising \eqref{tangent cone},  \eqref{velocity cone}, we can write the geometric condition that verifies whether there exists an admissible input for a state $x\in\cX$ that allows it to remain in $\cX$ as follows
\begin{equation}\label{eq: safe condition}
\cB(x) \cap \cF(x) \neq \emptyset.
\end{equation}
In our setting, the boundary of $\cX$ is defined either on $\cV^u$ or $\cV^l$. Let us consider the equations $C^u(x_1)$ \eqref{upper ACC} and $C^l(x_1)$\eqref{lower ACC}. The functions are semi-differentiable.
Taking into account dynamics \eqref{eq:state dynamics forward},  each vector  $y$ in the set $\cF(x)$ has a nonnegative first element, i.e., $y_1=x_2\geq 0$. Consequently, in the set \eqref{velocity cone} it is sufficient to consider only the \emph{right-derivative} of $C^u(x_1)$ and $C^l(x_1)$ in order to verify \eqref{eq: safe condition}. We write the right derivative of the upper and lower constraint curves as
\begin{align}
        m_u(y_1)= \nabla^+(C^u(y_1)) \label{upper right},\\
        m_l(y_1) = \nabla^+(C^l(y_1)) \label{lower right}.
\end{align}

This is computationally convenient, allowing us to define a half-space including the set of vectors from \eqref{tangent cone} that can possibly intersect \eqref{velocity cone}. The resultant half-spaces can be explicitly defined as

\begin{align}
    \cB^u (y) &= \{ x\in \R^2: x_2 \leq ( m_u(y_1) x_1 -y_1 ) - y_2 \}\label{tagent cone u},\\
    \cB^l (y) &= \{ x\in \R^2: x_2 \geq ( m_l(y_1) x_1 -y_1 ) - y_2 \}\label{tagent cone l}.
\end{align}

Thus, condition \eqref{eq: safe condition} can be verified by checking  \eqref{eq:practical safe set} 
\begin{equation}\label{eq:practical safe set}
    \cB(x) \cap \cF(x) = \begin{cases}
        \cB^u(x) \cap \cF(x)\; \text{when}\; x\in\cV^u,\\
        \cB^l(x) \cap \cF(x)\; \text{when}\; x\in\cV^l.
    \end{cases}
\end{equation}
We define the function $S(\cdot): \cV \to \R$, 

\begin{equation} 
    S(x) = 
    \begin{cases}
        \begin{bmatrix} -m_u(x) & 1 \end{bmatrix}^\top f(x, 0), & \text{when} \;\;\; x \in \cV^u,\\
        \begin{bmatrix} m_l(x) & -1  \end{bmatrix}^\top  f(x, 1), & \text{when} \;\;\; x \in \cV^l.
    \end{cases}\label{s(x)}
\end{equation}
\noindent
where $m_u(x)$ and $m_l(x)$ are given by \eqref{upper right} and \eqref{lower right}, respectively.

\begin{lemma}\label{equivalent_condition_lemma}
Consider the system \eqref{eq_system_state_equations} and the input and state constraints \eqref{eq:admissible and input constraints}. Then, for  any $x\in\cV^u$ \eqref{upper ACC}, or any $x\in\cV^l$ \eqref{lower ACC}, it holds
\begin{equation*}
S(x)\leq 0 \Leftrightarrow  \cB(x) \cap \cF(x) \neq \emptyset.
\end{equation*}
\end{lemma}

We consider a partition that will be later used to characterise the constraint curve where \eqref{eq: safe condition} is satisfied.

\begin{lemma}\label{finite L lemma}
Consider the system \eqref{eq_system_state_equations}, the constraints \eqref{eq:admissible and input constraints}, and the boundaries $\cV=\left\{\cV^u, \cV^l \right\}$ where  $\cV^u$ and $\cV^l$ are defined by \eqref{upper ACC}  \eqref{lower ACC}, respectively.
Consider the function $S(x)$ \eqref{s(x)}. Then, for each $\cV\in \{ \cV^u, \cV^l\}$, there is an integer $M\geq 1$ and a finite set of mutually disjoint intervals $\cI_{in}= \{ \cI_{in}(1), ...,\cI_{in}(M) \}$, $\cup_{i=1}^M \cI_{in}\subseteq [0,1]$, where it holds $S(x)\leq 0$ for all $x\in\cI_{in}$.
\end{lemma} 

\begin{definition}\label{interval definition}
Consider the system \eqref{eq_system_state_equations}, constraints \eqref{eq:admissible and input constraints} and an interval $\begin{bmatrix} (x_{end})_1, (x_{start})_1  \end{bmatrix}$. 

We partition the interval in an ordered set of intervals $\cI = \{\cI(1), ..., \cI(L) \}$,  $\cI(i)\subseteq [(x_\text{end})_1, (x_\text{start})_1]$, $i=1,..,L$, constructed from two sets, namely, $\cI_\text{in} = \{ \cI_{\text{in}}(1), ..., \cI_{\text{in}}(M)\}$ and $\cI_{\text{out}} = \{ \cI_{\text{out}}(1), ..., \cI_{\text{out}}(P)\}$. Moreover, $|\cI_\text{in}|+|\cI_\text{out}|=M+P=L$, with $\cI_{\text{in}}(i)\neq \cI_{\text{out}}(j)$, for all $i=1,...,M, j=1,...,P$ and 
\begin{align*}
    x\in(\cI_{\text{in}})_i & \Rightarrow \cB(x) \cap F(x,u(x)) \neq \emptyset,  \\
    x\in(\cI_{\text{out}})_j & \Rightarrow \cB(x) \cap F(x,u(x)) =\emptyset. 
\end{align*}
\end{definition}

We note any two intervals $\cI_i$, $\cI_{i+1}$ are adjacent to each other and $\cI_i\cap\cI_{i+1}=(x_{\text{end}})_i=(x_{\text{start}})_{i+1}$. Moreover, for any $i=1,..,L-1$, if $\cI_i\in\cI_{\text{in}}$, then $\cI_{i+1}\in\cI_{\text{out}}$ and vice versa. 

\begin{ilex} \label{ilex_2}
Figure \ref{fig: intervals} focuses on the part enclosed by the rectangle in Figure \ref{alg1 e}. It illustrates a partition according to Definition \ref{interval definition}, showing how the section in $\cV^l$ is extended. We represent $\cI_{in}$ with $\cI_{out}$ with green and red dotted lines respectively. The cones generated by \ref{velocity cone}, evaluated at two states $x_{\text{C}}, x_{\text{D}} \in\cV^l$, are represented with purple vectors. The halfspaces $\cB^u({x_{\text{C}}})$, $\cB^u({x_{\text{D}}})$ \ref{tagent cone u}, which in this case coincide with the tangent cones $\cB({x_{\text{C}}})$ and $\cB({x_{\text{D}}})$ \eqref{tangent cone} are shown in grey. We observe  that for any $x\in\cV^l$ in an interval $\cI\in\cI_{\text{in}}$, the system can avoid crossing the boundary. On the other hand, for any $x\in\cV^l$ in  $\cI\in\cI_{\text{out}}$, there is no control action preventing crossing $\cV^l$.
\end{ilex}

As highlighted in Example \ref{ilex_2}, the boundary $\cV^u$ or $\cV^l$ often does not belong in its entirety to the reach avoid set, thus, a modification is needed. The procedure, called \texttt{extend}, is presented in Algorithm 2. Given the set of intervals for which condition $S(x)\leq 0$ holds as in Lemma 4, Algorithm 2 either adds the whole interval to the set $\cZ$ or propagates the closed-loop dynamics backwards in order to form a curve that will serve as the upper or lower bound of the computed set in Algorithm 1.  

Lines 5 and 17 of Algorithm 2 separate the cases where $\cI\in\cI_\text{in}$ and $\cI\in\cI_\text{out}$. When $\cI\in\cI_\text{in}$ (Line 9), the entire interval is added to the reach avoid set boundary. Lines 11-14 add a parts of the generated trajectory and existing constraint curve defined by $\cV$.  Line 16 represents the case where only an extreme trajectory is added to the set $\cZ$.
When $\cI\in\cI_\text{in}$ (Line 19) an $\epsilon$ step change from the boundary is performed, leading to the trajectory generation in Line 20. 
The following central result establishes that Algorithm 1 returns a reach-avoid set, in finite time.

\begin{algorithm}[h]\label{algorithm 2}
\caption{\texttt{extend}($\cV$, $[ (x_{\text{end}})_1, (x_{\text{start}})_1]$, $\lambda$, $\epsilon$)}
\begin{algorithmic}[1]
\State $\mathbf{Input:}$ Set $\cV\in\{\cV^l,\cV^u\}$  \eqref{upper ACC}, \eqref{lower ACC}, start state $x_{\text{start}}$, end state $x_{\text{end}}$, $\epsilon>0$, $\delta=(-1)^{\lambda+1}\epsilon$, partition $\cI$ as in Definition 5.
\State $\mathbf{Output:}$ $\cZ$
\State $\cZ \leftarrow \emptyset$,  $i \leftarrow 1$, $y \leftarrow \{x\in\cV:x_1 =(x_\text{start})_1\}$

\While{$i\leq L$}
    \If{$\cI(i)\in\cI_{\text{in}}$}
        \If{ $y\in \cV$}
            \State $\cZ \leftarrow \cZ \cup \cW(\cV,\cI(i))$
        \Else
            \State $\cT_\cI = \cW(\cT_b(y,\lambda),\cI(i))$
            \If{ $\cT_\cI \cap \cV \neq \emptyset$}
                \State $ x_{\text{int}} \leftarrow \cQ(\cT_\cI \cap \cV)$
                \State $\cT_{\text{int}} \leftarrow \cT_\cI\cap \{x \in \cT_\cI : x_1 \geq x_{\text{int}} \}$
                \State $\cV_{\text{int}} \leftarrow  \{x \in \cV : x_1 \leq x_{\text{int}}, x \in \cI(i) \}$
                \State $\cZ \leftarrow \cZ \cup \cT_{\text{int}} \cup \cV_{\text{int}}$
            \Else
                \State $\cZ \leftarrow \cZ \cup \cT_\cI$
            \EndIf
        \EndIf
    \ElsIf{$\cI(i)\in\cI_{\text{out}}$}
        \If{ $ y \in \cV$}
            \State $y \leftarrow \begin{bmatrix} y_1 & y_2 + \delta \end{bmatrix}^\top$
        \EndIf
        \State $\cT_\cI = \cT_b(y, \lambda) \cap \{x\in \R^2: (\exists l\in\cI(i):  x_1=l_1)\}$
        \State $\cZ \leftarrow \cZ \cup \cT_\cI$  
    \EndIf
    \State $i \leftarrow i + 1$
    \State $y \leftarrow \cQ(\cZ)$
\EndWhile
\State \textbf{return} $\cZ$
\end{algorithmic}
\end{algorithm}

\begin{figure}[h]
    \centering
    \begin{tikzpicture}[thick,scale=3.0, every node/.style={scale=1.0}, every path/.style={rfill}]
    \clip (-0.3,-0.3) rectangle + (2.8,1.75); 
    \draw[->,thin, black] (0,0)--(2.2,0) node[right]{$x_1$};
    \draw[->,thin, black] (0,0)--(0,1.3) node[above, xshift=-0.1cm]{$x_2$};
    
    \fill [black!20!white, rotate=24] (0.8, 0.15) rectangle (1.5,0.65) node[pos=.5, black, scale=0.6] {$\cB^u(x_\text{C})$};`
    
    \fill [white, rotate=24] (1.4,0.65) ellipse (0.2cm and 0.05cm);  
    \fill [black!20!white, rotate=24] (1.0,0.64) ellipse (0.2cm and 0.05cm);

    \fill [black!20!white, rotate=13.5] (1.5, 0.44) rectangle (2.08, 0.94) node[pos=.5, black, scale=0.6] {$\cB^u(x_\text{D})$};`
    
    \fill [white, rotate=13.5] (1.95,0.95) ellipse (0.15cm and 0.05cm);    
    \fill [black!20!white, rotate=13.5] (1.65,0.93) ellipse (0.15cm and 0.05cm);
    
    \draw[-,very thick] (0.7,0.475)--(1.3,0.74);
    \draw[-,very thick] (1.4,0.78)--(1.9,0.9);
    
    \draw[constraint_color, thick] (0.05, 3) .. controls (1.5, 4) and (2.25, 1) .. (3,4) node[above, xshift=-1.2cm]{$\cV^{u}$};
    \draw[constraint_color, thick] (3, 4) .. controls (4, 6) and (5, 3) .. (6,5);
    \draw[constraint_color, thick] (6, 5) -- (6, 0.05);
    \draw[constraint_color, thick] (0.05, 0.05) .. controls (3, 2) and (3.5, 0.05) .. (6, 0.05)node[above, xshift=-11cm, yshift=2.1cm]{$\cV^{l}$};
    \draw[constraint_color, thick] (0.05, 0.05) -- (0.05, 3);
    \draw[red!90!black, thick] (6, 2) -- (6, 3)node[right, yshift=-0.5cm, xshift=0cm]{$\cX_\text{T}$};
    \fill[black,thick,dashed] (6,2) circle (0.05cm) node[right, yshift=-0.25cm, xshift=0cm]{$x^{*}_l$};
    \fill[black,thick,dashed] (6,3) circle (0.05cm) node[right, yshift=0.25cm, xshift=0cm]{$x^{*}_u$};
    
    \draw[magenta!70!black, thick] (6, 3) .. controls (5, 4) and (2, 1) .. (0.05,2) node[above, yshift=1cm, xshift=4.5cm]{$T^{*}_u$};
    \fill[black,thick,dashed] (0.05,2) circle (0.05cm) node[left, yshift=0cm, xshift=0cm]{$x_d$};
    
    \draw[magenta!70!black, thick] (6, 2) .. controls (5, 1) and (3, 2) .. (2,0.9) node[above, yshift=0.8cm, xshift=1cm]{$T^{*}_l$};
    \fill[black,thick,dashed] (2,0.9) circle (0.05cm) node[right, yshift=0.2cm, xshift=0.3cm, scale = 0.75]{$x_a$};

    \draw[red!85!black, dashed, ultra thick] ((1.4, 0.73) ..controls (1.6, 0.81) and (1.8, 0.83) .. (2, 0.88);
    \draw[green!90!black, dashed, ultra thick] ((0.6, 0.36) ..controls (0.67, 0.4) and (0.8, 0.5) .. (1.4, 0.73);
    \draw[red!85!black, dashed, ultra thick] ((0.05, 0.03) ..controls (0.2, 0.13) and (0.4, 0.25) .. (0.6, 0.36);
    
    \draw[black, dashed, thick](2,-0.05)--(2,0.9);
    \draw[black, dashed, thick](1.4,-0.05)--(1.4,0.8);
    \draw[black, dashed, thick](0.6,-0.05)--(0.6,0.4);  
    \draw[black, dashed, thick](0.05,-0.05)--(0.05,0.05);
        
    \draw[ red!85!black,->,thin] (1.45,-0.05)--(1.95,-0.05)node[below, scale=0.6, xshift=-1.0cm, yshift=0.0cm]{$\cI(1)$};
    \draw[ red!85!black,->,thin] (1.95,-0.05)--(1.45,-0.05);
    
    \draw[green!70!black,->,thin] (0.65,-0.05)--(1.35,-0.05)node[below, scale=0.6, xshift=-1.5cm, yshift=0.0cm]{$\cI(2)$};
    \draw[green!90!black,->,thin] (1.35,-0.05)--(0.65,-0.05); 
    
    \draw[red!85!black,->,thin] (0.1,-0.05)--(0.55,-0.05)node[below, scale=0.6, xshift=-0.9cm, yshift=0.0cm]{$\cI(3)$};
    \draw[red!85!black,->,thin] (0.55,-0.05)--(0.1,-0.05);  
    
    \fill[black,thick,dashed] (1.0,0.6) circle (0.02cm) node[below, yshift=-0.15cm, xshift=0.1cm, scale=0.75]{$x_{\text{C}}$};
    
    \draw[->,thick, magenta!70!black](1.0,0.6)--(1.2,0.8)node[above, yshift=-0.8cm, xshift=0.2cm, scale=0.6]{$\cF({x}_\text{C})$};
    \draw[->,thick, magenta!70!black](1.0,0.6)--(1.2,0.5);
    
    \fill[black,thick,dashed] (1.65,0.84) circle (0.02cm) node[below, yshift=-0.1cm, xshift=-0.1cm, scale=0.75]{$x_{\text{D}}$};
    
    draw[->,thin](1.65,1.15)--(1.65,0.98);
    
    \draw[->,thick, magenta!70!black] (1.65,0.84)--(1.75,0.6)node[below, yshift=0.45cm, xshift=0.4cm, scale=0.6]{$\cF({x_D})$};
    \draw[->,thick, magenta!70!black] (1.65,0.84)--(1.9,0.8)node[below, yshift=0.0cm, xshift=-0.0cm, scale=0.6]{};
    \end{tikzpicture}
    \caption{Illustration of a typical set of intervals obtained along the lower velocity limit curve in the region to be extended from Figure \ref{alg1 e}. The red regions $\cI(1), \cI(3)\in\cI_{out}$ and the green regions $\cI(2)\in\cI_{in}$. At both states $x_\text{C}$, $x_\text{D}$, the extremes of $\cF(\cdot)$ are represented by purple vectors, and lower boundary of the halfspace $\cB(\cdot)$ is shown as a tangent to $\cV^l$ at the respective states.}\label{fig: intervals} 
\end{figure}
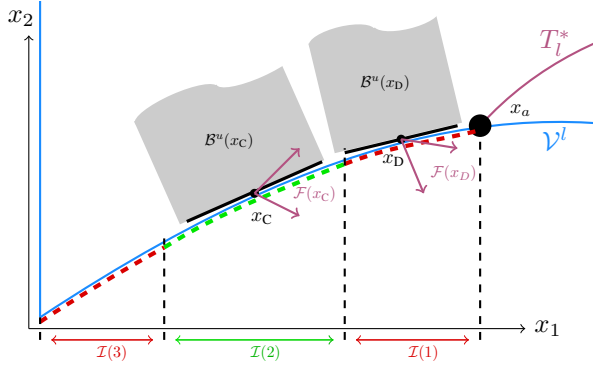

\begin{theorem} \label{theorem_1}
Consider the system  \eqref{eq:state dynamics forward}, constraints \eqref{eq:admissible and input constraints} and the target set $\cX_{\text{T}}$ \eqref{target set}. The following hold.
\begin{itemize}
    \item[(i)] Algorithm 1 terminates in finite time.
    \item[(ii)]
    The set $\cR_\epsilon(\cX_T)$ is a reach-avoid set.
\end{itemize}
\end{theorem}

\section{Controller design}\label{section 7}
We turn our attention to the control problem, namely, constructing safe state-feedback mechanisms. First, we characterise the generic properties the controller should have so that any state $x$ in $\cR_\epsilon(\cX_T)$ can be transferred to the target set $\cX_T$ in finite time without violating state and input constraints. 
Let us denote the upper and lower boundary of the reach avoid set by $\cZ_u$ and $\cZ_l$ respectively.  

\begin{theorem}\label{control design theorem}
Consider the system \eqref{eq_system_state_equations} with state and input constraints $\cX$, $\cU_x$  \eqref{eq:admissible and input constraints}, and a target set $\cX_\text{T}$. Consider a reach avoid set $\cR_\epsilon (\cX_\text{T})$ and an initial condition $x_0\in \cR_\epsilon (\cX_\text{T})$, and the control law  $u(x, \lambda(x) )$ \eqref{eq: input definition}. Then, the trajectory $\cT_f(x_0, \lambda(x))$  lies within $\cR_\epsilon (\cX_\text{T})$ and reaches $\cX_\text{T}$ in finite time if 
\begin{equation}\label{lambda x cases}
\lambda(x) =  
\begin{cases}
    1 & \forall x \in \cZ_l, \\
    0 & \forall x \in \cZ_u,\\
    c(x) & \forall x \in \interior{(\cR_\epsilon(\cX_\text{T}))},\\
\end{cases}
\end{equation}
and $\lambda(\cdot):\cX\rightarrow [0,1]$ is a Lipschitz continuous function.
\end{theorem}

The control parameterisation above is quite general. A possible realisation of \eqref{lambda x cases} is to consider two locally Lipschitz functions $g_1(x)$ and $g_2(x)$, with $g_1(x_1) >  g_2(x_1)$,  $x\in\cX$. The continuous actuation level function is 

\begin{equation}\label{control guide saturation equation}
\lambda(x) =  
\begin{cases}
    0 & \text{if $x_{2}\geq g_{1}(x_1)$} \\
    1 & \text{if $x_{2} \leq g_{2}(x_1)$} \\
    \frac{x_{2} - g_{1}(x_1)}{g_{2}(x_1) - g_{1}(x_1)} & \text{if $g_2(x_1) < x_{2} < g_1(x_1)$}
\end{cases}
\end{equation}
We note that if we use the upper and lower bounds of $\cR_\epsilon(\cX_T)$ to define $g_1(x_1)$ and $g_2(x_1)$ respectively, a valid controller can be defined as $\lambda(x) = \frac{x_2 - x_2^u}{x_2^l - x_2^u}$.

This is a convex linear mapping between the upper an lower bounds.

It is straightforward to see that this choice satisfies all conditions \eqref{lambda x cases} of Theorem \ref{control design theorem}, excluding the Lipschitz continuity of $\lambda(x)$. Indeed, due to the formation of the bounds detailed in Algorithm 1, there may be discontinuities if steps are used in the production of the boundaries. A practical way to make these bounds continuous, and therefore utilise the aforementioned simple controller, is to  approximate the reach-avoid set bounds by (piecewise) polynomial functions.

Additionally, our framework captures a family of \emph{sliding mode-like controllers } \cite{edwards1998sliding}.  

To elaborate, a special case in \eqref{control guide saturation equation} can be retrieved by  setting $g(x_1) = g_1(x_1) = g_2(x_1)$, leading to the simplified actuation level function
\begin{equation}\label{bang bang control method}
\lambda(x) =  
\begin{cases}
    1 & \text{if $x_2 < g(x_1)$} \\
    0 & \text{if $x_2 \geq g(x_1)$}
\end{cases}
\end{equation}

We note the resulting controller does not present a continuous control but rather a bang bang style controller, that is prone to chattering \cite{lynch_park}. However, it can be the basis for a sliding mode control strategy:  We may define the function $H(x, \lambda)=|\begin{bmatrix} -m(x) & 1 \end{bmatrix}^\top  f(x, \lambda) |$, 
 where $f(x, \lambda)$ represents the system dynamics as given in Definition \ref{def: input dynamics cone} and $m(x) = \frac{d g(x_1)}{d x_1}$. By a suitable choice of $g(x_1)$, if a parameter $\lambda$ can be found such that $H(x, \lambda) = 0$ for all $x$,  the controller meets all conditions for sliding mode operation by choosing
 \begin{equation}\label{improved bang bang control method}
\lambda(x) =  
\begin{cases}
    1  &\text{if $x_2 < g(x_1)$}, \\
    0 &\text{if $x_2 > g(x_1)$}, \\
    \argmin \lambda  H(x, \lambda) &\text{if $x_2 = g(x_1)$.}
\end{cases}
\end{equation}

\section{Numerical experiment}\label{section 8}
In this section we illustrate the established framework, in terms of reach-avoid set computation and control implementation. Subsections VII.A, VII.B deal with practical considerations in implementation, while subsection VII.C provides the simulation results.

\subsection{Digital implementation}\label{sec: trajectory generation}
We consider the issue of sampling in the implementation of the state feedback and computation of  trajectories forming the reach avoid set (Algorithms 1, 2). The solution of system \eqref{eq_system_state_equations} piecewise constant inputs $u_T$ is analytic, that is, for any $t\in [0, T]$ we have
$    \phi_\text{f}(t;x_0,u(x))=e^{At} x_{0}  + u_T \int_{0}^{t} e^{A\tau} B d \tau,$

or,
    $x(t) = \begin{bmatrix}
        x_{1}(t)\\
        x_{2}(t)
    \end{bmatrix} = 
    \begin{bmatrix}
        \frac{t^2}{2} u_T + t(x_0)_2 + (x_0)_1\\
         t u_T + (x_0)_2
    \end{bmatrix}.$

Sampling time should be chosen to be small enough in order to approximate sufficiently a continuous-time controller. Input $u_T$ can be enforced to be admissible in $[0,t]$, since $u(x, \lambda(x))$ is by construction Lipschitz continuous in $x$. To this purpose,  for any two states
$x_0, x(t)\in \cX$ and considering the Lipschitz constants $L_1$, $L_2$ so that  $\|A(x_0) - A(x(t)) \| \leq L_1 \| x_0-x(t)  \|$ and $\|D(x_0) - D(x(t)) \| \leq L_2 \| x_0-x(t)  \|$, the following inequalities should hold 
\begin{equation} \label{eq_sample}
D(x) + L_2 \| x_0-x(t) \|\leq u_T \leq A(x_0) - L_1 \| x_0-x(t) \|.
\end{equation}

\subsection{Approximation of Constraint Curves}

The partition $\cI$ in Lemma \ref{finite L lemma} requires finding the roots of the function $S(x)$ \eqref{s(x)}.
By continuity of $S(x)$, this can be done solving conditions $S(x)=0$ and $\left\|\frac{d(S(x))}{dx}\right\| = \left\| \lim \limits_{h \to 0} \frac{S(x + h [1 \ \ m]^\top) - S(x)}{h}\right\| \neq 0\label{condition dS(x) neq 0}$,

with $m = m_u$ if $x$ is in $\cV^u$ and $m = m_l$ if $x$ is in $\cV^l$.
Solving above equations requires non-straightforward application of numerical methods on nonlinear functions, while in practice boundary values are evaluated at a finite number of states rather than providing analytical expressions \cite{8988702}. 
To tackle this challenge, we can approximate  $S(x)$ with a  polynomial function $\tilde{S}(x_1)$ so that

$\tilde{S}(x) \geq S(x)$ for all states $x$ belonging in $\cV^u$ and $\cV^l$,  by solving a constrained least squares problem. 

To improve approximation accuracy, one can typically increase the order of the approximating polynomial, or calculate piecewise polynomial approximations. 

\subsection{Results}\label{simulation example}
We apply our framework to the two DOF planar manipulator setup  illustrated in Figure \ref{fig-a}-\ref{fig-c}, and the commercially available Universal robot UR5  \cite{kebria2016kinematic}.
The UR5 simulation data are detaild below, where the two DOF manipulator setting is in Appendix \ref{sim data}.
The UR5 manipulator consists of six joints and respective actuators, for which $\tau^{\text{max}} = \begin{bmatrix} 150 & 150 & 150 & 28 & 28 & 28
    \end{bmatrix}$ and $\tau^{\text{min}} = \begin{bmatrix}
    -150 & -150 & -150 & -28 & -28 & -28 \end{bmatrix}$.
We define a straight line path in the joint space as $q(s) = q(0) - s (q(0) - q(1))$ where $q(0) = \begin{bmatrix}\frac{\pi}{2} & -\frac{\pi}{4} & \frac{\pi}{3} & \frac{2\pi}{3} & -\frac{\pi}{2} &
-\frac{\pi}{3}\end{bmatrix}^\top$, $q(1) = \begin{bmatrix}
    0 & 0 & 0 & 0 & 0 & 0
\end{bmatrix}^\top$.
The torque limitations along with the path description allow the computation of the path dynamics and  $A(x)$, $D(x)$ \eqref{eq:minA}, \eqref{eq:maxD}. We assume additional inequality constraints translated in the state space $x_2 \leq 4 \sin(10x_1+5) - 2 \sin(18x_1^3) + 10$, $x_2 \geq 4 \sin(10x_1+5) - 2 \sin(18x_1^3) - 2$.
We apply piecewise constant inputs $u_T$, following the analysis Section VII.A. 
To compensate for the lack of knowledge of the Lipschitz constant for the state feedback, we consider modified maximum acceleration and deceleration bounds, so that the controller is defined as $u(x, \lambda(x)) = \tilde{D}(x) - \lambda(x)(\tilde{A}(x) - \tilde{D}(x))  $,  where $\tilde{A}(x) = A(x) -  0.05 (A(x) - D(x))$ and $\tilde{D}(x) = D(x) +  0.05 (A(x) - D(x))$.
The modified input equation creates a buffer of 5\% the difference between the limits to allow for the case where $A(x)$ may decrease or $D(x)$ may increase over the sampling time, and is consistent with the effect the Lipschitz constant has on the bounds \eqref{eq_sample}.

The target set \eqref{target set} is $\cX_{\text{T}} = \{ x \in \cX: x_1=1, 1 \leq x_2 \leq 5 \}$. 
The shaded region in Figure \ref{fig:RAS and constraints} represents the reach avoid set calculated via Algorithm 1, with the target set denoted in red. In this instance, we identify two intervals (Definition 5) for the lower bound and 167 intervals for the upper bound with the value of $\epsilon$ set to 0.1 for both boundaries.

\begin{figure}[h]
    \centering
    \begin{tikzpicture}
    \node[anchor=south west,inner sep=0] at (0,0) {\includegraphics[width=1\linewidth]{"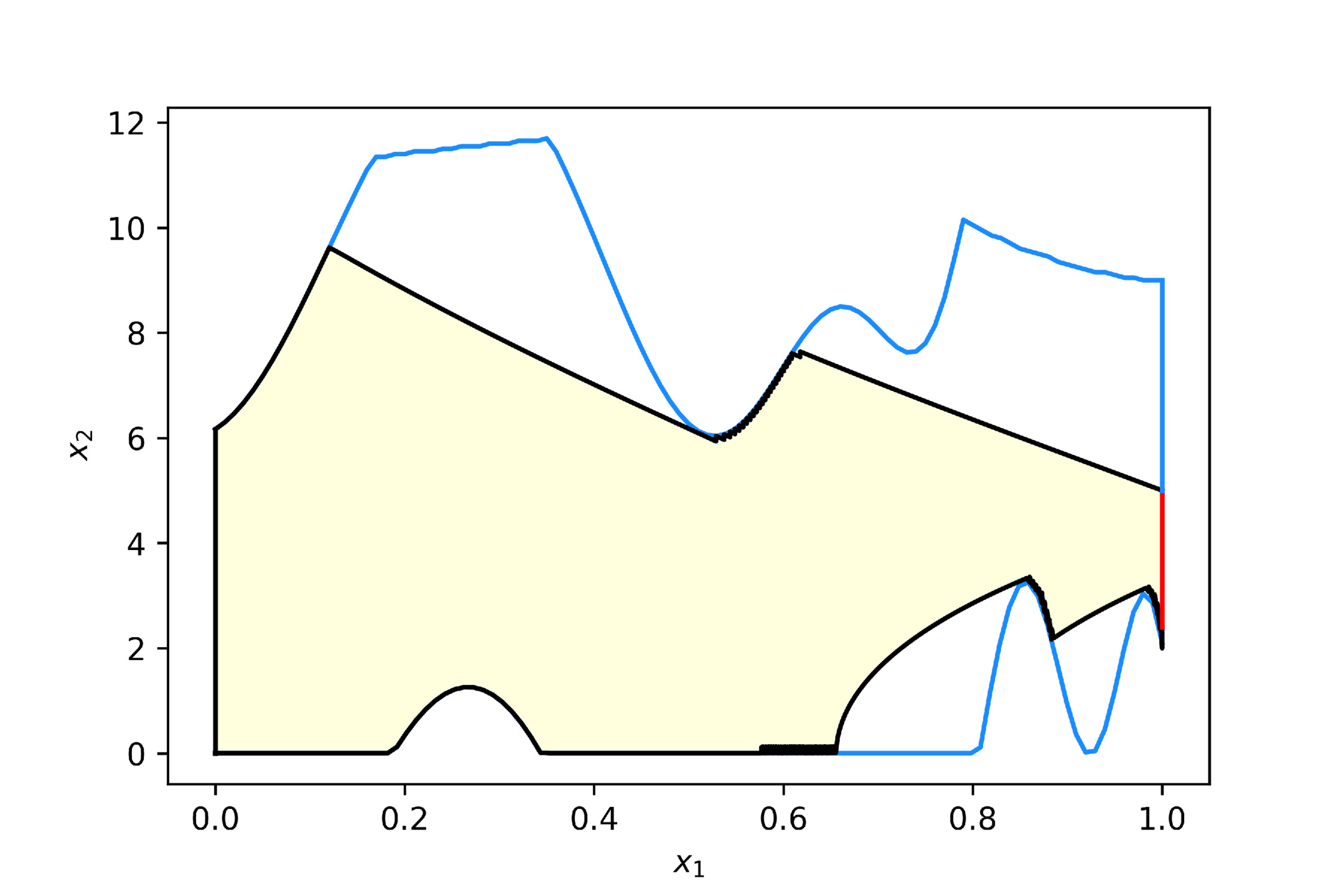"}};
    \node[text width=3cm] at (4.5,4.25) 
    {$\cX$};
    \node[text width=3cm] at (5.0,2.25) 
    {$\cR_\epsilon(\cX_\text{T})$};
    \end{tikzpicture}

    \caption{$\cR_\epsilon(\cX_\text{T})$ is shown in shaded color. The boundary of $\cX$ is in blue,  the target set $\cX_\text{T}$ is in red.}\label{fig:RAS and constraints}
\end{figure}

To showcase a specific control trajectory, we set $x_{\text{initial}} = \begin{bmatrix} x_1 & x_2 \end{bmatrix} = \begin{bmatrix} 0 & 3 \end{bmatrix}$. 
We apply two different control strategies: The first concerns the \emph{state feedback} controller described by the actuation level function in \eqref{control guide saturation equation} with
$g_1(x_1) = 2.92 x_1^2 + -3.42 x_1 + 5$, $g_2(x_1) = -4.78 x_1^2 + 8.18 x_1$.
The second is the \emph{ open-loop time optimal} control \cite{lynch_park}. To compute the time optimal path we target a final state of $x_{\text{final}} = \begin{bmatrix} x_1 & x_2 \end{bmatrix}^\top = \begin{bmatrix} 1 & 4 \end{bmatrix}^\top$.
\begin{figure}[h]
    \centering
    \begin{tikzpicture}
    \node[anchor=south west,inner sep=0] at (0,0) {\includegraphics[width=1\linewidth]{"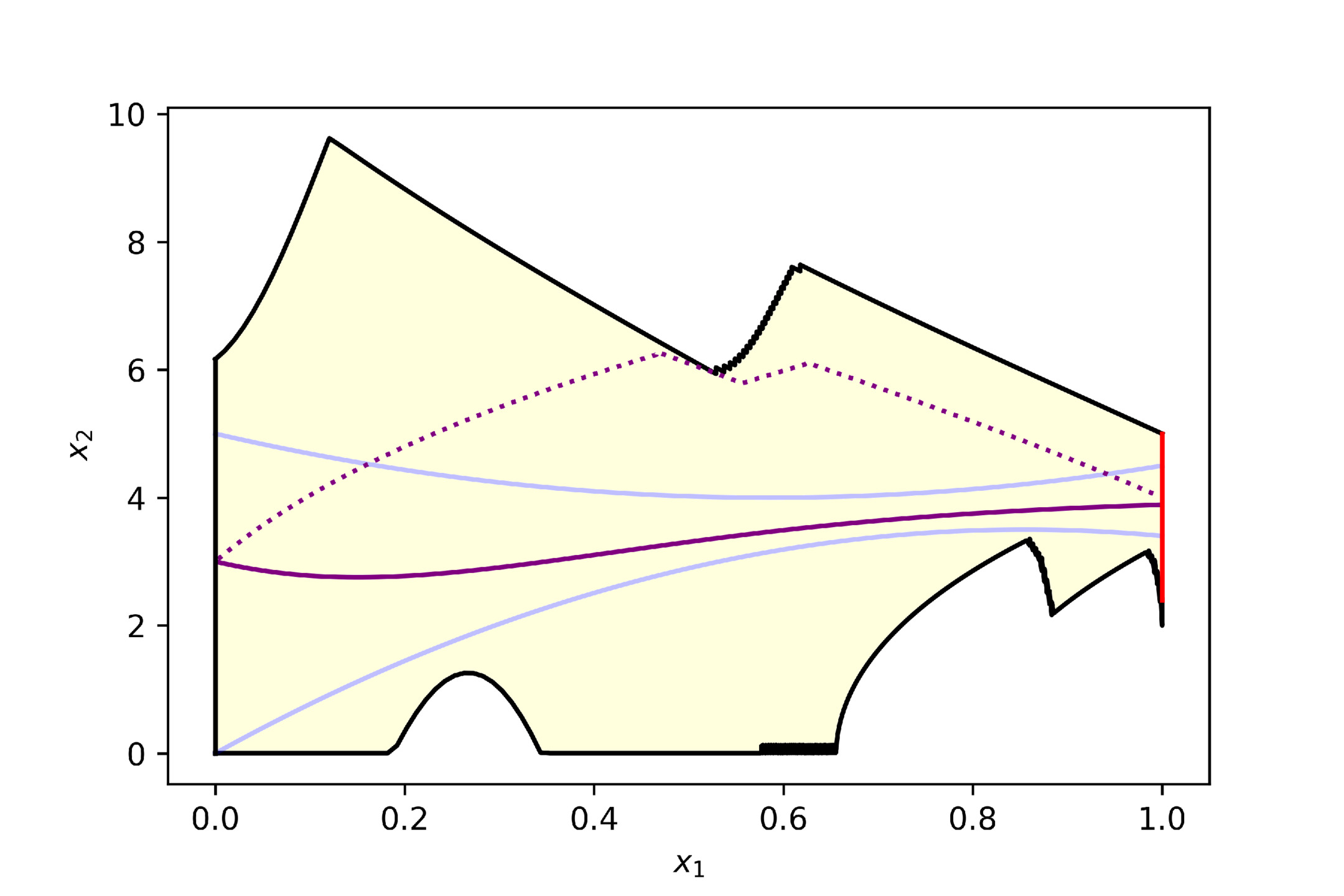"}};
    \node[text width=1cm, scale=0.6, color=purple, rotate=5] at (3.0,2.35) 
    {$\text{state feedback}$};
    \node[text width=1cm, scale=0.6, color=purple, rotate=20] at (2.8,3.15) 
    {$\text{Time optimal}$};
    \node[text width=1cm, scale=0.6, color=blue!60!white, rotate=-10] at (2.0,3.15) 
    {$g_{1}(x_{1})$};
    \node[text width=1cm, scale=0.6, color=blue!60!white, rotate=15] at (2.0,1.5) 
    {$g_{2}(x_{1})$};    
    \end{tikzpicture} 
    \caption{Time optimal control trajectory (dotted purple), state feedback control trajectory (solid purple with inputs held for 10ms), $g_1(x_1)$ and $g_2(x_1)$ are (light blue) .}\label{fig:time optimal vs state feedback}
\end{figure}
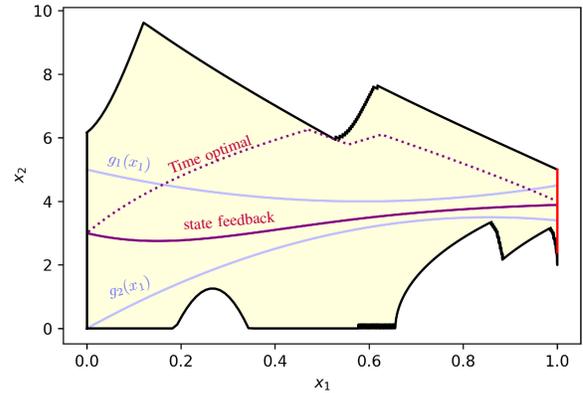
The closed loop trajectories of the two chosen controllers is shown in Figure \ref{fig:time optimal vs state feedback}. It is worth observing the reach avoid set contains both the time optimal trajectory as well as the state feedback controller. Moreover, we note the form of $g_{1}(x_{1})$ and $g_{2}(x_{1})$ can be varied to affect the shape of the state feedback trajectory. 
The computation times for each case are in Table I. Both control implementation and simulations are written in python, and ran on a standard laptop with a 7th gen intel core i5 processor. 

The potential of our approach is shown in the average computation time for a single step of the feedback control strategy. It is also worth noting that the computation time for the reach avoid set on the UR5 added to the trajectory generation is faster than computing the time optimal control alone.

\setlength\extrarowheight{3pt}
\begin{table}[H]
\centering

\begin{tabular}{l|l|l}
Robot                              & \textbf{UR5} & \textbf{Two DOF Planar}  \\ 

\hline
Time optimal controller                              &  $82.4$ s & $2.3 \times 10^{-1}$ s              \\ 
\hline
Controller \eqref{control guide saturation equation} computation time                       & $49 \times  10^{-3}$ s  & $21  \times  10^{-6}$ s              \\
\hline
Sampling period T for \eqref{control guide saturation equation}               & $50  \times 10^{-3}$ s  & $1 \times 10^{-3}$  s              \\
\hline
Reach avoid set computation                     & 60.3 s  & 9.6 s              \\

\hline
\end{tabular}
\caption{Table presenting data for time optimal and state feedback computations based on the reach avoid set.}
\end{table}

\noindent
The time response of the UR5 in the joint space is shown in Figure \ref{fig:joint response}. Unsurprisingly, the time optimal control is faster.
\begin{figure}[h]
    \centering
    \includegraphics[width=1\linewidth]{"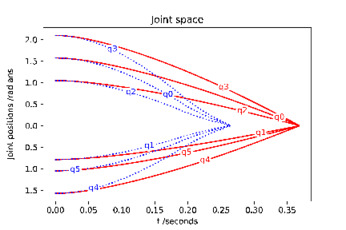"}
    \caption{Minimum time optimal joint response  (dotted blue) and state feedback \eqref{control guide saturation equation} response (solid red). }\label{fig:joint response}
\end{figure}
The torque response for the time optimal and the state feedback methods are given in Figures \ref{fig:TO torque response}, \ref{fig:fb torque response} respectively. It is clear that the time optimal control leads to switching between maximum and minimum acceleration for at least one actuator, exhibiting large instantaneous variations. In this instance there are two switches, however in the general case there can be several, contributing to the wearing of actuators and possibly leading to additional steps for  shaping the control input.
In contrast, the state feedback response gives a naturally Lipschitz continuous response through the application of equation \eqref{control guide saturation equation} in choosing the actuation level to apply. A real time implementation is indeed possible using conventional computing resources as illustrated in Figure \ref{fig:50 torque response}. Nevertheless, the control response can be made smoother with a smaller sampling period as shown in Figure \ref{fig:10 torque response},  which would be implementable, e.g., using embedded hardware implementation. 

\begin{figure}[h]
    \centering
    {\includegraphics[width=1\linewidth]{"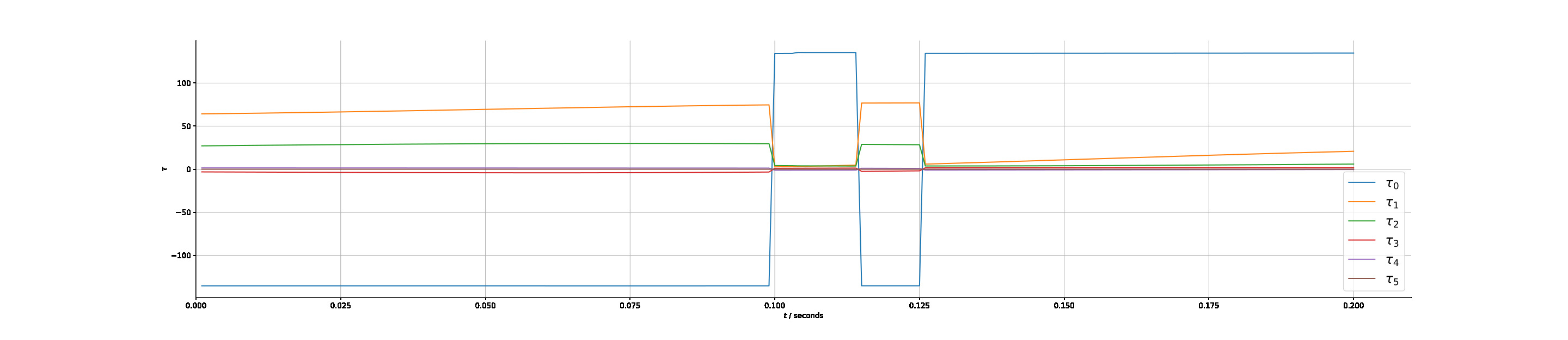"}};
    \caption{Optimal control torques ($\tau_0,...,\tau_5$) time response.}\label{fig:TO torque response}
\end{figure}

\begin{figure}[h]
    \centering
    \begin{subfigure}{\linewidth}
    {\includegraphics[width=1\linewidth]{"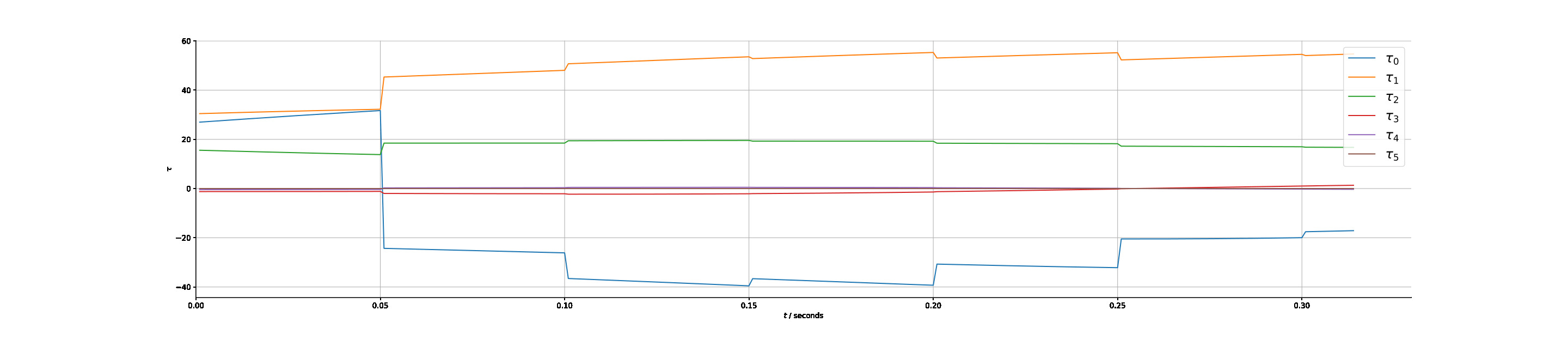"}};
    \subcaption{Torques ($\tau_0,...,\tau_5$) time response, $T=50$ms.}\label{fig:50 torque response}
    \end{subfigure}
    
    \begin{subfigure}{\linewidth}
    {\includegraphics[width=1\linewidth]{"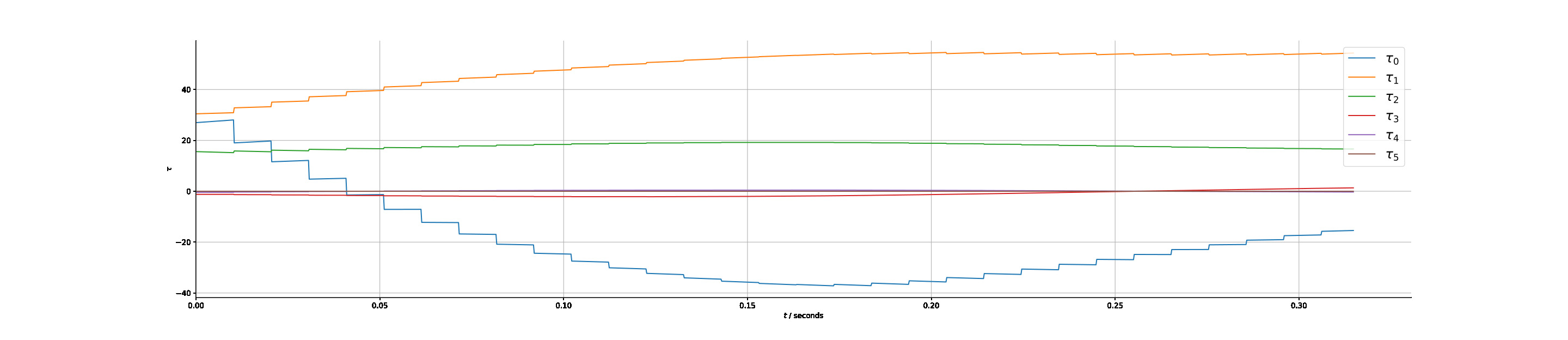"}};
    \subcaption{Torques $\tau_0,...,\tau_5$ time response, $T=10$ms.}\label{fig:10 torque response}
    \end{subfigure}
    
    \caption{State feedback control \eqref{control guide saturation equation} torque response for each actuator.}\label{fig:fb torque response}
\end{figure}

\section{Conclusions} \label{Conclusions}
We explored the trajectory planning problem for robotic manipulators using a well established methodology from a different angle.
We formulated and addressed the problem of finding safe state-feedback controllers as a reach-avoid problem in the projected path dynamics using ordering properties of closed-loop trajectories under a suitable parameterisation of the control law. The established algorithm is practicable and  terminates in finite time, while numerical experiments with a commercial manipulator show the method works well.  The feedback mechanism as well the scalability of the approach suggest that the method could be potentially be used in problems in safety, e.g., in collaborative robotics.  Our future work will focus on using the established framework as a building block on composition of trajectories between different paths with motion primitives, using the inherent flexibility offered, with the ultimate aim to address the path planning problem. Moreover, we aim to relax some of the assumptions on the shapes of the constraint and target sets posed herein, effectively increasing the generality of our approach.

\bibliographystyle{IEEEtran}	
\bibliography{references}

\begin{appendices}

\section{Manipulator models}\label{sim data}
The simple two DOF planar manipulator shown in Figure \ref{fig-a} used throughout the examples has the Lagrangian dynamic parameters 
\begin{align*}
    M_L &= \begin{bmatrix}
        0.03 + 0.02 \cos(q_2) & 0.01 + 0.01 \cos(q_2)\\
        0.01 + 0.01 \cos(q_2) & 0.01
    \end{bmatrix}\\
    C_L &= \begin{bmatrix}
        0 & -0.01(2\dot{q_1} + \dot{q_2}) \sin(q_2)\\
        0.01\dot{q_1} \sin(q_2) & 0
    \end{bmatrix}\\
    g_L &= \begin{bmatrix}
        0.981 \cos(q1) + 0.4905 \cos(q_1+q_2)\\
        0.4905 \cos(q_1+q_2)
    \end{bmatrix}
\end{align*}

We consider a circular arc path that gives us is defined by the center point and radius. The workspace description in  $x$ and $y$ directions is encoded by $z(x_1)$ and $y(x_1)$ as defined by a radius of $0.15$ and center point $(0.1, 0.3)$, i.e.,  $z(x_1) = 0.1 + 0.15cos(\pi x_1)$, $y(x_1) = 0.3 - 0.15sin(\pi x_1)$,
The variation of the joints in the form $q(x_1)=\begin{bmatrix}q_1(x_1) & q_2(x_1) \end{bmatrix}^\top$, with
\begin{align*}
    q_1(x_1) &= \arctan2(y(x_1),z(x_1)) -\\ \nonumber & \arccos(\frac{z^2(x_1)+y^2(x_1)}{0.4\sqrt{z^2(x_1)+y^2(x_1)}})\\
    q_2(x_1) &= \pi - \arccos(1-12.5(x^2(x_1)+y^2(x_1))),
\end{align*}
allowing $q_1(x_1)$ and $q_2(x_1)$ to be written in terms of $x_1$ as
follows
\begin{align*}
    &q_1(x_1) = \arctan2(0.3 - 0.15\sin(\pi x_1) ,0.1 + 0.15\cos(\pi x_1)) -\\ \nonumber & \arccos(\frac{(0.1 + 0.15\cos(\pi x_1))^2+(0.3 - 0.15\sin(\pi x_1) )^2}{0.4\sqrt{(0.1 + 0.15cos(\pi x_1))^2+(0.3 - 0.15\sin(\pi x_1) )^2}}),\\
    &q_2(x_1) = \pi -  \arccos(1-12.5((0.1 + 0.15\cos(\pi x_1))^2\\ \nonumber & \;\;\;\; +(0.3 - 0.15\sin(\pi x_1) )^2)). 
\end{align*}

The relevant vectors for the path dynamics can be computed as $M(x_1)\dot{x_2} = M_L\frac{dq(x_1)}{d x_1}$, $C(x_1)x_2^2 = M_L\frac{d^2q(x_1)}{d x_1^2} x_2^2 + C_L \dot{q}(x_1),$ $g(x_1) = g_L(q(x_1))$

The times presented in section \ref{simulation example} concern the setting described in this section. We also apply the same constraint vector 
with the target set defined as \eqref{target set} in $\cX_{\text{T}} = \{ x \in \cX: x_1=1, 4 \leq x_2 \leq 8.5 \}$. The time optimal method uses the final state $x_{\text{final}} = \begin{bmatrix} x_1 & x_2 \end{bmatrix} = \begin{bmatrix} 1 & 4 \end{bmatrix}$.
The feedback mechanism is of the form \eqref{control guide saturation equation} with $ g_1(x_1) = 2 x_1 + 4$, $g_2(x_1) = 3.5 x_1 + 1$.

The initial state of the simulation is $x_{\text{initial}} = \begin{bmatrix} x_1 & x_2 \end{bmatrix}^\top = \begin{bmatrix} 0 & 3 \end{bmatrix}^\top$.

The universal robots UR5s mathematical model can be found in \cite{kebria2016kinematic}, \url{https://www.mathworks.com/matlabcentral/fileexchange/72049-kinematic-and-dynamic-modelling-of-ur5-manipulator?s_tid=srchtitle}.

\section{} \label{app_proofs}

\subsection{Proof of Proposition \ref{proposition 1}}

The following preliminary Fact establishes results between compositions of Lipschitz continuous functions, while Lemma \ref{lemma 1} shows that at least one element of $M_i(x_1)$ is strictly positive for all admissible values of $x_1$. 

\begin{Fact} \label{fact_lipschitz_combination}
 Suppose functions $f_1(\cdot):\R^n\rightarrow \R$ and $f_2(.):\R^n\rightarrow \R$ are locally Lipschitz continuous in a set $\cS\subset\R^n$, with Lipschitz constants $L_1$ and $L_2$ respectively. The following hold:
\begin{itemize}
    \item[(i).] $f_3(x) = f_1(x)+f_2(x)$ is locally Lipschitz in $\cS$.
    \item[(ii).] $f_4(x)=f_1(f_2(x))$ is locally Lipschitz in $\cS$.
    \item[(iii).] Suppose that $f_1(x)$, $f_2(x)$ are bounded in $\cS$. Then, $f_5(x)=f_1(x)f_2(x)$ is locally Lipschitz in $\cS$. 
    \item[(iv).] Suppose that $|f_1(x)|>a$, $x\in\cS$. Then the function $f_6(x)=\frac{1}{f_1(x)}$ is locally Lipschitz in $\cS$.
    \item[(v).] Consider $f_i(x)$, $i=1,...,N$ and the function $f(x)=\max_i\{ f_i\}$. Let $f_i(x)$ be Locally Lipschitz in $\cS_i=\{ x\in\cS: f_i(x)\geq f_j(x), \forall j=1,..,N, j\neq i\}$. Then, $f$ is locally Lipschitz in $\cS$.
    \item[(vi).] Consider $f_i(x)$, $i=1,...,N$ and the function $f(x)=\min_i\{ f_i\}$. Let $f_i(x)$ be locally Lipschitz in $\cS_i=\{ x\in\cS: f_i(x)\leq f_j(x), \forall j=1,..,N, j\neq i\}$. Then, $f$ is locally Lipschitz in $\cS$.
\end{itemize}
\end{Fact}

\begin{proof}
(i) For $x,y\in\cS$, we have $|f_1(y)+f_2(y)-f_1(x)-f_2(x) |\leq |f_1(y)-f_1(x) |+|f_2(y)-f_2(x) |\leq (L_1+L_2)|y-x|$. (ii) We have $|f_1(f_2(y))-f_1((f_2(x)) |\leq L_1|f_2(y)-f_2(x)|\leq L_1L_2|y-x|.$ (iii). Let  $|f_1(x)|\leq M_1$, $|f_2(x)|\leq M_2$ for all $x\in\cS$. It follows $|f_1(y)f_2(y)-f_1(x)f_2(x) |=|f_1(y)f_2(y)-f_1(y)f_2(x)+f_1(y)f_2(x)+f_1(x)f_2(x) |\leq M_1L_2|y-x| +M_2L_1|y-x|=(M_1L_2+M_2L_1)|y-x|$. 
(v) We can write $\max(f_1, f_2) = \frac{f_1 + f_2 +| f_1 -f_2|}{2}$. Since $|x|$ is Lipschitz continuous it can be concluded from Fact 1 (ii) that $| f_1 -f_2|$ is also Lipschitz.  (vi) Since  $\min(f_1, f_2) = \frac{f_1 + f_2 -| f_1 -f_2|}{2}$, the same reasoning as Fact (v) can be applied.  $\blacksquare$
\end{proof}

\begin{lemma}\label{lemma 1}
Under Assumption 1, for all $x\in \cX^{'}$ there exists at least one index $i^{*}\in\{1,2,...,N \}$ such that $|| M_{i^{*}}(x_1)||\geq \alpha_{1}$, for some $\alpha_1>0$.
\end{lemma}

\begin{proof}
By construction, it holds that $M(x_1)=M_L(x_1)\frac{dq(x_1)}{dx_1}$. Moreover, the symmetric mass matrix $M_L(x_1)$ is positive definite for all $x \in \cX'$. 
Let us consider an arbitrary $y=\frac{d q(x_1)}{d x_1}$, such that $||y||=\alpha_0$, where $\alpha_0$ is provided in Assumption~1. Let $z_i$, $i = 1,...,N$, be the normalised eigenvectors of $M_L(x_1)$ such that $||z_i|| = \alpha_0, i = 1,..,N$. The eigenvectors $z_i$ are mutually orthogonal as $M_L(x)$ is symmetric. Consequently, we can express $y$ in the base of eigenvectors $y = \sum^{N}_{i=1}c_i \hat{z}_i$, with  $c_i\geq 0$, $\sum_{i=1}^Nc_i\geq 1$ and $\hat{z}_i=z_i$, or $\hat{z}_i=-z_i$. The last inequality holds since $\|y\|=\|z_i\|$. It follows that $M_L(x_1)y=\sum_{i=1}^N\lambda_ic_iz_i$, where $\lambda_i>0$ are the eigenvalues of $M_L(x_1)$. Consider the set $\cS = \text{conv}(\pm  z_1,..,\pm  z_N)$. We can write $\cS=\{x\in\R^N: \|Px \|_1\leq 1 \}$, where $P=\left(\begin{bmatrix} z_1 & \cdots & z_N \end{bmatrix}^{\top}\right)^{-1}$.  
Setting $d = \sum_{i=1}^N \lambda_ic_i $ we can write $M_L(x_1)y=d \sum_{i=1}^N \left( \frac{\lambda_ic_i}{d} \right)z_i$, or, $M_L(x_1)y\in d\boundary{\cS}.$  By norm equivalence we have $\|x\|_1\leq N\| x\|_2$, thus $d\cS\supseteq \B(\frac{1}{dN\sigma_{\max}(P)})$, and $\|M_L(x_1)y\|\geq \frac{1}{d N \sigma_{\max}(P)}$, where $\sigma_{\max}(P)$ is the largest singular value of $P$. Consequently, there necessarily exists at least one element of $M_L(x_1)y$ such that $\|M_L(x_1)y \|\geq \alpha_1$ for some finite $\alpha_1$ and for all $y\in \B(\alpha_0)$. The result for $\| y\|>\alpha_0$ follows directly by the homogeneity of the linear mapping.   $\blacksquare$
\end{proof}

\noindent \textbf{Proof of Proposition \ref{proposition 1}}
The proof is split into two parts, namely when the elements $|M_i(x)|$ are bounded from below for all $i$, and  where there exists a subset of indices $\hat{I}\subset[1,N]$, where $|M_i(x)| \leq \epsilon$ for any arbitrarily small $\epsilon$.

\noindent 
\underline{Case 1:} $| M_i(x) | \geq \alpha_1>0$, $i=1,...,N$.
By definition \eqref{eq:minA}, \eqref{eq:maxD} and Fact 1 (v), (vi), $A(x)$ and $D(x)$ are locally Lipschitz continuous in $\cX_i$ if $A_i(x)$ and $D_i(x)$ are. We show that $A_i(x)$ and $D_i(x)$ are Lipschitz functions by analysing the numerator and the denominator separately.
The numerator in $A_i(x)$ and $D_i(x)$ is constructed by the sum of the elements of vectors $\tau^{\max}$ and $\tau^{\min}$, $-C(x)x_2^2$ and $-g(x_1)$. By Fact 1 (i), (iii) if each vector element in this sum is Lipschitz the numerator will be Lipschitz for all $i=1,...,N$. By Assumption \ref{ass_actuator_lims}, $\tau^{\max}(x)$ and $\tau^{\min}(x)$ are Lipschitz.
It is also known from, e.g., \cite{choi2001iterative}, \cite{bouakrif2013velocity},  \cite{de2001commande}, that $M_L(\cdot)$, $C_L(\cdot)$ and $g_L(\cdot)$ are locally Lipschitz functions.  The vector $g(x_1) = g_{L}(q(x_1))$, therefore $g(x_1)$ is also Lipschitz from Fact 1 (ii). 
The vector $C(x)x_2^2$ can be rewritten as $C(x)x_2^2= M_L(q(x_1)) \frac{d^2(q(x_1))}{d x_1^2}x_{2}^{2} + C_L (q(x1), \dot{q}(x_1))$. Since $M_L(\cdot)$, $\frac{d^2(q(x_1))}{d x_1^2}$, $x_{2}^{2}$, $C_L (\cdot)$ are all locally Lipschitz continuous, by applying Fact 1 (i), (iii) we have that  $C(x)x_2^2$ is locally Lipschitz continuous. Hence, the numerator in the definition of equations $A_i(x)$, $D_i(x)$ is locally Lipschitz. 
The denominator for both $A_i$ and $D_i$ is $M_{i}(x_1)$. From Fact1 (iii) $M(x_{1}) = M_L(q(x_{1})) \frac{d q(x_1)}{d x_{1}}$ is Lipschitz continuous since  $M_{L}(q(x_1))$ and $\frac{d q(x_{1})}{d x_{1}}$ are. Since $| M_{i}(x_1) | > \alpha_{1}$, by Fact 1 (iv) $\frac{1}{M_{i}(x_1)}$ is also Lipschitz. Thus, $A_{i}$ and $D_{i}$, and consequently $A(x)$ and $D(x)$ are locally Lipschitz continuous in $\cX^{'}$.

\noindent
\underline{Case 2:} Suppose there exists a subset of indices $\hat{\cI}\subset \{1,2,...,N\}$ such that $| M_{i} (x) | \leq \epsilon$, where $\epsilon$ is an arbitrary small real number and $i\in\hat{I}$. 
By Lemma \ref{lemma 1}, there exists at least an index i such that $|M_i(x)| \geq \alpha_1$, for some $\alpha_1 > 0$. Let $\cI = \{1, 2,...,N\} \setminus \hat{\cI}$. Let us consider a point where $M_i(x_1)$, $i\in\hat{\cI}$, becomes arbitrarily smal.  
Setting $a_i(x) = \frac{\tau^{\max}_{i}(x) - C_i(x_1)x_2^2 - g_i(x_1)}{M_i(x_1)}$, 
$d_i(x) = \frac{\tau^{\min}_{i}(x) - C_i(x_1)x_2^2 - g_i(x_1)}{M_i(x_1)}$

we can express $A_i(x)$ and $D_i(x)$ as
\begin{align}
A_i(x) = 
\begin{cases}
    a_{i}(x) & \text{if $M_i(x_1)>0$},\\
    d_{i}(x) & \text{if $M_i(x_1)<0$},\\
\end{cases}\label{A(x) min cases}\\
D_i(x) = \begin{cases}
    d_{i}(x) & \text{if $M_i(x_1)>0$},\\
    a_{i}(x) & \text{if $M_i(x_1)<0$}. \\
\end{cases}\label{D(x) max cases}
\end{align}
Further, we set $a_i(x)=\frac{\beta}{\epsilon}$, $d_i(x)=\frac{\gamma}{\epsilon}$.

Assuming $\epsilon \to 0^{+}$ or $\epsilon \to 0^{-}$ we have
\begin{align}
\lim\limits_{\epsilon \to 0^{+}} (a_i) = 
\begin{cases}
    -\infty \;\;\; & \text{when} \;\;\; \beta < 0,\\
    +\infty \;\;\; & \text{when} \;\;\; \beta > 0,
\end{cases}\label{1}\\
\lim\limits_{\epsilon \to 0^{+}} (d_i) \to 
\begin{cases}
    -\infty \;\;\; & \text{when} \;\;\; \gamma < 0,\\
    +\infty \;\;\; & \text{when} \;\;\; \gamma > 0,
\end{cases}\label{2}\\
\lim\limits_{\epsilon \to 0^{-}} (d_i) = 
\begin{cases}
    +\infty \;\;\; & \text{when} \;\;\; \gamma < 0,\\
    -\infty \;\;\; & \text{when} \;\;\; \gamma > 0,
\end{cases}\label{3}\\
\lim\limits_{\epsilon \to 0^{-}} (a_i) \to 
\begin{cases}
    +\infty \;\;\; & \text{when} \;\;\; \beta < 0,\\
    -\infty \;\;\; & \text{when} \;\;\; \beta > 0.
\end{cases}\label{4}
\end{align}

Equations \eqref{1}, \eqref{3} relate to $A_i(x)$ and \eqref{2}, \eqref{4} relate to $D_i(x)$. We note for the  case where $\beta=0$, then $a_i = 0$, and when $\gamma=0$, then $d_i=0$, which is a Lipschitz continuous function. Last, we define
\begin{align*}
    A^{'}(x) &:= \min\limits_{\{i\in I\}}A_i(x), &
    \hat{A}(x) &= \min \limits_{\{i \in \hat{I}\}} A_i(x),\\
    D^{'}(x) &= \max\limits_{\{i\in I\}}D_i(x), &
    \hat{D}(x) &= \max \limits_{\{i \in \hat{I}\}} D_i(x).
\end{align*}
Functions $A^{'}(x), D^{'}(x) \in \R$ have a finite range. Consequently, $A(x)$ and $D(x)$ can be written as 
\begin{align}
    A(x) &= \min( A^{'}(x), \hat{A}(x)), \nonumber\\
    D(x) &= \max( D^{'}(x), \hat{D}(x)). \nonumber
\end{align}

We study the case when $\epsilon \to 0$. If $\hat{A}(x) \to \infty$, then $A(x) = A^{'}(x)$. Likewise, if $\hat{D}(x) \to - \infty$, then $D(x) = D^{'}(x)$. In these cases  we revert to Case 1.  Analysing similarly \eqref{1} - \eqref{4} allows us to observe that all possible situations that render $A(x)$ and $D(x)$ Lipschitz continuous are covered when $\beta \geq 0$ and  $\gamma \leq 0$. This holds in the cases where $\epsilon \to 0^{+}$ and $\epsilon \to 0^{-}$. Finally, we consider the case $\beta < 0, \gamma > 0$, in the cases where $\epsilon \to 0^{+}$ and $\epsilon \to 0^{-}$. We define two large positive numbers, $M_1$ and $M_2$ such that $A(x) = \min (A^{'}, -M_1) = -M_1$ and $D(x) = \max ( D^{'}, M_2 ) = M_2$. Clearly, $A(x)<D(x)$ thus, $x \notin \cX'$.  Hence, any state $x$ that produces an unbounded value for $A(x)$ or $D(x)$ is necessarily outside the admissible region, i.e., $x\notin \cX'$. We  conclude the only possible arrangements provide local Lipschitz continuity properties to \eqref{A(x) min cases} and \eqref{D(x) max cases}. By combining Case 1 and 2, the result is obtained. $\blacksquare$

\subsection{Proof of Lemma \ref{rule 1}}

By Corollary \ref{col: lambda x}, $u(x)$ is locally Lipschitz continuous in $\cX$. Consequently, by Fact 1 (i), (iii) the backward dynamics $f(x) = -Ax-Bu(x)$ is also locally Lipschitz continuous. Therefore by the Picard–Lindelöf theorem the solution of  $\dot{x}= f(x)$ is unique for a given $x$  \cite{9780821883280} \cite{lindelof1894application}. Suppose an intersection between two trajectories occurs at $x$, the backward dynamics must produce two solutions which is a contradiction. Therefore intersection is not possible. The same reasoning can be applied to the forward dynamics.
$\blacksquare$

\subsection{Proof of Lemma \ref{rule 2} }
Suppose there are two trajectories that intersect defined as $\cT_u = \cT_u(x_u, \lambda_u)$ and $\cT_l = \cT_l(x_l, \lambda_l)$ with $\cT_l \cap \cT_u = \{x^{*}\}$. We describe these two trajectories with four pieces  that emanate from $x^{*}$, namely, as $\cT_u = \cT_{b,u} \cup \cT_{f,u}$ and $\cT_l = \cT_{b,l}\cup \cT_{f,l}$, where $\cT_{b,u} = \cT_b(x^{*}, \lambda_u)$, $\cT_{f,u} = \cT_f(x^{*}, \lambda_u)$, $\cT_{b,l} = \cT_b(x^{*}, \lambda_l)$, $\cT_{f,l} = \cT_f(x^{*}, \lambda_l)$. 
We observe that for all $x\in int(\cX)$, $\frac{\partial u(x, \lambda)}{\partial \lambda} =A(x)-D(x)> 0$. Thus, 
 $\lambda_u > \lambda_l$ implies $u(x^{*}, \lambda_u) > u(x^{*}, \lambda_l)$. 
Let  $ \dot{x}=f(x)$ as given by \eqref{eq:state dynamics forward}. Consider the gradients defined as $f^u = f(x^{*}, \lambda_u) = \begin{bmatrix}x^*_2 & u(x^{*}, \lambda_u) \end{bmatrix}^\top$ and 
$f^l = f(x^{*}, \lambda_l) = \begin{bmatrix}x^*_2 & u(x^{*}, \lambda_l) \end{bmatrix}^\top$.
We consider the \emph{angle of emanation} from $x^\ast$ for some vector $f\in\R^2$ as $\theta(f) = \tan^{-1}{( \frac{f_2}{f_1}} )$. At  intersection it holds that $f^u_1=f^l_1$, and $f^u_2>f^l_2$, thus  $\frac{f^u_2}{f^u_1} > \frac{f^l_2}{f^l_1}$, or,  $\theta(f^u)>\theta(f^l)$. Consequently, for any $\delta>0$ small enough and the slice $\cX_f = \cW(\cX, x_1^* + \delta)$ with $\bar{x}_u=\cT_{f,u}\cap\cX_f$ and $\bar{x}_l=\cT_{f,l}\cap\cX_f$, it holds that $\bar{x}_{u,2} > \bar{x}_{l,2}$.
Similarly, for the  slice $\cX_b = \cX_v(x_1^* - \delta)$ with $\tilde{x}_u=\cT_{b,u}\cap\cX_b$ and $\tilde{x}_l=\cT_{b,l}\cap\cX_b$ it holds that $\tilde{x}_{u,2} < \tilde{x}_{l,2}$.  
The proof is complete if one considers that the trajectories are continuous, thus, no jumps are allowed and consequently a second intersection cannot happen.   
$\blacksquare$

\subsection{Proof of Lemma \ref{equivalent_condition_lemma} }

$(\Rightarrow)$ The inner product $S(x)$ is between two vectors, the first being the normal vector $n\in\left\{ \begin{bmatrix} -m_u(x) & 1 \end{bmatrix}^\top , \begin{bmatrix}m_l(x) & -1 \end{bmatrix}^\top \right\}$ to $\cV$ where $\cV\in \{\cV^u, \cV^l\}$ evaluated at $x\in\cV$.
In general $\cV$ is not smooth, however it is continuous.

The second part of the inner product represents  the vector field of the system dynamics choosing the extreme value of the control actuation $\lambda$, i.e, $\dot{x} \in \left\{ f(x, 0), f(x, 1)  \right\}\subset \cF(x,u(x))$.
Therefore, when $S(x)\leq 0$ then $\dot{x} \in \cF(x)$, and $\dot{x} \in \cB(x)$, thus $\cB(x) \cap \cF(x) \neq \emptyset$.
$(\Leftarrow)$
Since for each $x$ the input is a convex combination of two limits  \eqref{eq: input definition}, $\cF(x,u(x))$ is a convex polyhedral cone with generators $A(x)$ and $D(x)$. Consequently, $\cB(x) \cap \cF(x) \neq \emptyset$ necessarily implies  $\begin{bmatrix}x_2 & D(x)\end{bmatrix}^\top\cap \cB(x)\neq \emptyset$ when $x\in\cV^u$ and $\begin{bmatrix}x_2 & A(x)\end{bmatrix}^\top\cap \cB(x)\neq \emptyset$. Therefore $S(x)\leq 0$.~$\blacksquare$

\subsection{Proof of Lemma \ref{finite L lemma}}

We need to show that $S(x)$ has a finite number of roots, thus, the partition $\cI$ is finite. First, notice if there are any equilibrium points for the system \eqref{eq_system_state_equations}, they induce  intervals. At non equilibrium points, the terms $-m_u(x) x_2 + D(x)$ or  $m_l(x) x_2 - A(x)$ are locally Lipschitz continuous functions, with a bounded derivative. Therefore, within a finite interval, the number of roots is bounded, and thus the partition is finite. $\blacksquare$

\subsection{Proof of Theorem \ref{theorem_1}}
 (i) All operations within Algorithm 1 involve propagation of trajectories and have a finite computation time. To show that Algorithm 2 terminates in finite time, it is enough to observe that $L$ is finite by Lemma 4, and every operation described within the loop in Lines 6-26 is finite. 
\noindent (ii) Consider any initial condition state $x$ in $\cR_\epsilon(\cX_T)$. 
To show $\cX_{\text{T}}$ will be reached in finite time we explicitly construct function $\lambda(x)$: We consider the slice $\cW(\cR_\epsilon (\cX_{\text{T}}), x_1)$ given by \eqref{slice description}, and define the vectors $x^u(x)\in\cW(\cR_\epsilon (\cX_{\text{T}}), x_1)$,  $x^l(x)\in\cW(\cR_\epsilon (\cX_{\text{T}}), x_1)$  on the upper and lower boundary of $\cR_\epsilon (\cX_{\text{T}})$. We define the control law $\lambda(x) = \frac{x_2 - (x^u(x))_2}{(x^l(x))_2 - (x^u(x))_2}$. 
The control law is constructed such than when $x_2 = x_2^l$, $\lambda(x) = 1$ and thus $\dot{x}_2 = u(x, \lambda(x)) = A(x)$. Likewise, if $x_2 = x_2^u$, $\lambda(x) = 0$ and thus $\dot{x}_2 = u(x, \lambda(x)) = D(x)$. By construction of $\cR_\epsilon(\cX_{\text{T}})$ and Lemmas \ref{equivalent_condition_lemma}, \ref{finite L lemma}, if the state lies on the upper or lower boundary $\cV^u$ and $\cV^l$ respectively, then \eqref{eq: safe condition} is verified for some inputs and there exists an input that will drive the system inside $\cX$. Moreover, since $\dot{x}_1 =x_2 \geq 0$ the value of $x_1$ increases with time when $x_2>0$. For the case when $x_2=0$, $x$ is necessarily on the lower boundary, thus, $\dot{x}_2 = A(x) > 0$, thus, the value of $x_1$ will increase in finite time. Last, taking into account that the right boundary of   $\cR_\epsilon(\cX_{\text{T}})$ is $\cX_{\text{T}}$, there is necessarily a finite time $t^*$ such that the solution to the system $\phi_\text{f}(t^*;x_0,\lambda(x))$ is in $\cX_\text{T}$. $\blacksquare$

\balance

\subsection{Proof of Theorem \ref{control design theorem}}

The system \eqref{eq:state dynamics forward} is continuous if $\lambda(x)$  \eqref{lambda x cases} is Lipschitz continuous. Thus, $\cT_f(x, \lambda(x))$ represents a continuous curve in the phase plane. This prevents crossing of the boundary of $\cR_\epsilon(\cX_\text{T})$ without intersecting it.

Consider a state $x$ on the boundary $\boundary{\cR_\epsilon(\cX_{T})}$. Since $\cR_\epsilon(\cX_{T})$ is a reach-avoid set, there exists  $u(x, \lambda(x))$ such that \eqref{eq: safe condition} is verified.

Moreover, the value of $x_1$ will continually increase until $\cX_{T}$ is reached in finite time. Specifically, the time taken for a trajectory between any two points $(x_a)_1$ and $(x_b)_1$ is given by $T = \int_{(x_1)_a}^{(x_1)_b}\frac{1}{x_2(x_1)}d x_1$ \cite[Chapter 9]{lynch_park}.
$\blacksquare$

\end{appendices}

\end{document}